\documentclass[lettersize,journal]{IEEEtran}
\usepackage{cite}
\usepackage{amsmath,amssymb,amsfonts}
\usepackage{amsthm}
\usepackage{graphicx}
\usepackage{textcomp}
\usepackage{array}
\usepackage{textcomp}
\usepackage{stfloats}
\usepackage{url}
\usepackage{verbatim}
\usepackage{graphicx}
\usepackage{cite}
\usepackage[linesnumbered,ruled]{algorithm2e}
\usepackage{algorithmicx}
\usepackage{bbm}
\usepackage{subcaption}
\usepackage{mathtools}
\usepackage{graphicx}
\usepackage{booktabs}
\usepackage{xcolor}

\newtheorem{theorem}{Theorem}
\newtheorem{proposition}{Proposition}

\newtheorem{lemma}{Lemma}
\newtheorem{problem}{Problem}
\newtheorem{remark}{Remark}
\theoremstyle{definition}

\newcommand{\algname}{HardFlow}

\newif\ifmarked
\markedfalse
\newcommand{\revised}[1]{\ifmarked{\color{blue}#1}\else#1\fi}
\newenvironment{revisedblock}{\ifmarked\color{blue}\fi}{}

\def\BibTeX{{\rm B\kern-.05em{\sc i\kern-.025em b}\kern-.08em
    T\kern-.1667em\lower.7ex\hbox{E}\kern-.125emX}}
\begin{document}

\title{\algname: Hard-Constrained Sampling for Flow-Matching Models via Trajectory Optimization}

\author{Zeyang Li, Kaveh Alim, Navid Azizan\\[2.0ex]%
Massachusetts Institute of Technology%
    \thanks{Corresponding author: Zeyang Li (zeyang@mit.edu).}
}

\maketitle

\begin{abstract}
    Diffusion and flow-matching have emerged as powerful methodologies for generative modeling, with remarkable success in capturing complex data distributions and enabling flexible guidance at inference time. Many downstream applications, however, demand enforcing hard constraints on generated samples—for example, robot trajectories must avoid obstacles—a requirement that goes beyond simple guidance. Prevailing projection-based approaches constrain the entire sampling path to the constraint manifold, which is overly restrictive and degrades sample quality. In this paper, we introduce a novel framework that reformulates hard-constrained sampling as a trajectory optimization problem. Our key insight is to leverage numerical optimal control to steer the sampling trajectory so that constraints are satisfied precisely at the terminal time. By exploiting the underlying structure of flow-matching models and adopting techniques from model predictive control, we transform this otherwise complex constrained optimization problem into a tractable surrogate that can be solved efficiently and effectively. Furthermore, this trajectory optimization perspective offers significant flexibility beyond mere constraint satisfaction, allowing for the inclusion of integral costs to minimize distribution shift and terminal objectives to further enhance sample quality, all within a unified framework. We provide a control-theoretic analysis of our method, establishing bounds on the approximation error between our tractable surrogate and the ideal formulation. Extensive experiments across diverse domains, including robotics (planning), partial differential equations (boundary control), and vision (text-guided image editing), demonstrate that our algorithm, which we name \textit{\algname}, substantially outperforms existing methods in both constraint satisfaction and sample quality.
\end{abstract}

\begin{IEEEkeywords}
    Flow matching, diffusion, training-free guidance, constrained sampling, constrained optimization, optimal control, trajectory optimization.
\end{IEEEkeywords}

\section{Introduction}
Diffusion \cite{sohl2015deep, ho2020denoising} and Flow-matching \cite{lipman2023flow} have revolutionized generative modeling, achieving tremendous performance across diverse domains, including image synthesis \cite{rombach2022high, esser2024scaling}, video generation \cite{ho2022video, jin2025pyramidal}, molecule design \cite{watson2023novo}, and robotics \cite{chi2023diffusion, ding2025fast}. At their core, these methods learn a time-varying ``dynamics'' (e.g., score function, velocity field) that transports samples from a simple prior distribution (e.g., Gaussian noise) to a complex target data distribution. The training process is notably efficient, leveraging conditional probability paths to formulate a conditional loss that enables simulation-free learning. During inference, samples are drawn from the simple prior and then evolved along trajectories defined by the learned dynamics, a process governed by an ordinary differential equation (ODE) or a stochastic differential equation (SDE). Terminal states of the trajectories constitute samples from the target distribution. A key feature of these models is their flexibility, which allows for guidance during inference to steer samples toward regions of interest, typically by injecting the gradient of a differentiable objective into the learned dynamics \cite{dhariwal2021diffusion, ho2022classifier}.

While gradient-based guidance is effective for encouraging desirable \textit{soft} attributes, many practical applications present a more stringent challenge: enforcing \textit{hard} constraints, i.e., inviolable conditions that a sample must satisfy to be considered valid. For example, in robotic planning, a proposed trajectory must be collision-free or it may lead to catastrophic failures; in facial-expression image editing, the subject's identity should remain invariant as the expression changes. Moreover, recent work has also highlighted the utility of explicit hard constraints as an effective mechanism to prevent reward overoptimization in text-to-image diffusion models \cite{zhang2025aligning}.

To date, the predominant approach to hard-constrained sampling for diffusion and flow-matching models has been projection-based methods and their variants \cite{power2023sampling, christopher2024constrained, romer2025diffusion, yuan2023physdiff, xiao2025safediffuser, zhang2025constrained, luan2025projected, zampini2025training, bouvier2025ddat}. The most straightforward application of projection is a single, post-hoc correction where the final generated sample is projected onto the feasible set \cite{power2023sampling}. This simple approach, however, can often induce a significant shift away from the target data distribution \cite{christopher2024constrained}. To mitigate this, common refinements involve applying a projection within the sampling process. One prominent strategy is to project iteratively at each step, enforcing feasibility throughout the entire generation path \cite{christopher2024constrained, luan2025projected, zampini2025training, romer2025diffusion}. An alternative method is to ignore or relax the projections in early sampling steps, motivated by the insight that early-stage samples resemble noise and that tightly constraining these noisy precursors can be detrimental to final sample quality \cite{bouvier2025ddat, yuan2023physdiff, xiao2025safediffuser, zhang2025constrained, yang2025safeflowmatcher}.

Despite some empirical success, projections during sampling suffer from several fundamental limitations. First, these methods are intrinsically conservative as they enforce pathwise feasibility, a condition that is unnecessarily restrictive since only the terminal state (i.e., the sample of interest) must satisfy the constraint. The intermediate iterates are more of algorithmic artifacts. By imposing constraints on these transient states, projection during sampling prematurely prunes the search space, which can prevent the generative process from discovering higher-quality solutions that would otherwise be accessible. Although per-step projection can be interpreted as projected gradient ascent on the data log-likelihood \cite{christopher2024constrained}, this view relies on Langevin-style sampling with specific noise and step-size schedules and does not generally extend to arbitrary SDE discretizations or the ODE-based samplers used in flow matching. Second, these projection-based methods solely address constraint satisfaction and are not designed to optimize a reward function. In practical applications, however, it is often desirable not only to enforce hard constraints but also to optimize rewards to improve sample quality. An ideal method should jointly address both aspects and produce feasible samples of high quality.

\begin{figure*}[htbp]
    \centering
    \includegraphics[width=\textwidth]{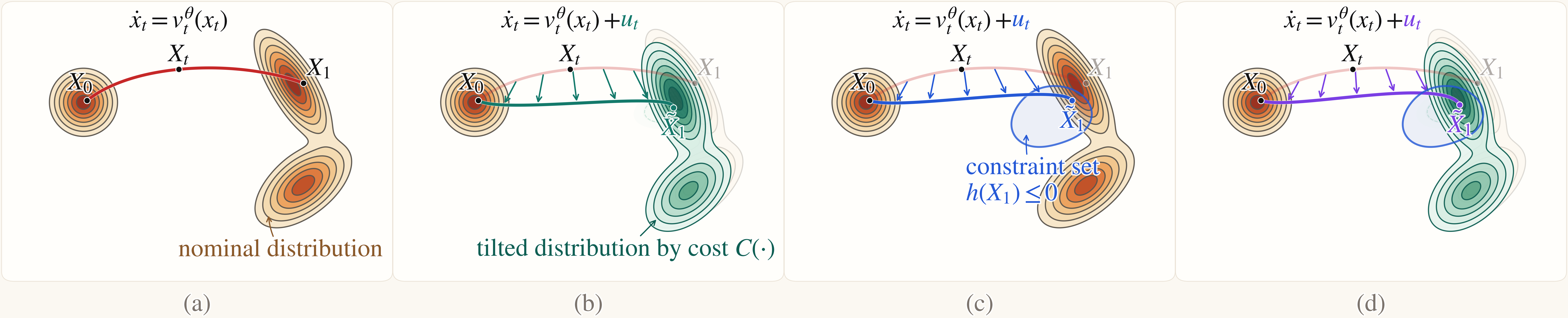}
    \caption{\revised{Conceptual illustration of our method. We adopt a trajectory optimization perspective to iteratively steer the sampling process toward the desired outcome. Panel (a) shows the nominal case: a pretrained flow-matching model transports an initial sample \(X_0\) from a simple base distribution through the ODE to the nominal distribution. Panel (b) shows the cost-only case, where the goal is to sample from a cost-tilted distribution defined by a cost function \(C(\cdot)\). Panel (c) shows the constraint-only case, where the goal is to generate samples in the feasible set defined by \(h(\cdot)\le 0\). Panel (d) shows the joint case, where the goal is to satisfy the constraint while also achieving low cost.}}
    \label{teaser}
\end{figure*}

To address the aforementioned challenges, we propose a novel framework that reformulates hard-constrained sampling as a trajectory optimization problem. While in this work we focus our formulation and analysis on ODE sampling for flow-matching models, the methodology is applicable to diffusion models under ODE sampling as well, such as denoising diffusion implicit models (DDIM) \cite{song2021denoising}. We consider perturbations to the learned velocity field as control inputs to steer the sampling trajectory. Constraints are then imposed solely on the terminal state, providing a minimally invasive correction that ensures the final sample lies within the feasible set without unnecessarily constraining the path. Furthermore, this framework naturally accommodates the inclusion of auxiliary objectives. We consider two primary types: (i) integral costs that penalize control effort to minimize the distributional shift from the uncontrolled sampler, and (ii) terminal costs designed to promote desirable attributes in the final sample. The trajectory optimization perspective thus offers a principled and unified framework that jointly addresses constraint satisfaction, distribution consistency, and sample quality. \revised{A conceptual illustration of our method is provided in Fig. \ref{teaser}.}

We note that prior work has connected diffusion and flow-matching models with optimal control for training-free guidance \cite{rout2025rb, wang2025training}. However, these approaches focus exclusively on guiding the samples toward high-reward regions and do not address hard constraints. Methodologically, they are grounded in Pontryagin's Maximum Principle (PMP) \cite{kirk2004optimal}, the cornerstone of \textit{indirect} methods in optimal control \cite{von1992direct}. This class of techniques is notoriously difficult to apply to problems with state constraints, thereby hindering the extension of \cite{rout2025rb, wang2025training} to hard-constrained sampling. In contrast, our work adopts \textit{direct} methods \cite{von1992direct} such as collocation \cite{conway2012survey} to discretize the dynamics and formulate a constrained optimization problem, enabling the explicit enforcement of hard constraints.

Nonetheless, the optimization problem is particularly challenging to solve. First, the number of decision variables is the product of the data dimension and integration steps, which is computationally demanding. Second, terminal constraints have to be propagated backward through equations involving the neural network dynamics (i.e., the learned velocity field). This creates an exceptionally complex feasible set, causing off-the-shelf solvers to struggle to find a feasible solution, let alone an optimal one. To address these issues, we leverage the specific structure of flow-matching models and draw inspiration from model predictive control (MPC) \cite{morari1999model} to derive a tractable surrogate, which leads to an efficient and scalable algorithm. The contributions of this paper are summarized as follows.
\begin{itemize}
    \item We introduce a novel framework that reformulates hard-constrained sampling for flow-matching models as a trajectory optimization problem. This perspective departs from predominant projection-based methods by enforcing constraints solely at the terminal time, thereby avoiding unnecessary restrictions on the sampling path. The framework also accommodates auxiliary objectives to minimize distribution shift and enhance sample quality.
    \item We develop a scalable algorithm, named \algname, to solve the formulated problem through principled transformations. First, with MPC principles, we decompose the long-horizon problem into a sequence of single-step subproblems. This is enabled by the flow-matching structure, which allows us to formulate effective proxies for terminal costs and constraints with the posterior mean. Second, we introduce a reverse reparameterization that changes the decision variable from the current state to the predicted terminal state, thereby avoiding the implicit, highly complex feasible set induced by propagating terminal constraints backward through the neural dynamics. This reparameterization is mathematically equivalent to a single-step fixed-point iteration that approximates an inverse mapping.
    \item We provide a control-theoretic analysis of our method, quantifying the approximation error introduced by our tractable surrogate relative to the original problem. These results clarify the algorithm's objective and how it balances computational efficiency with formulation optimality.
    \item We conduct extensive experiments across diverse domains, including robotics (planning), partial differential equations (boundary control), and vision (text-guided image editing). Our method consistently outperforms existing approaches in both constraint satisfaction and sample quality.
\end{itemize}

\section{Related Work}

This section situates our work within the relevant literature. We start by covering prior work in detail on constrained sampling for diffusion and flow-matching models—the core challenge addressed in this paper. As the advantages and limitations were comprehensively discussed in the Introduction, we do not repeat them here. Next, we provide a brief overview of optimal control taxonomy, as our approach is grounded in this field. We then discuss existing approaches that exploit optimal control for generative modeling. Finally, we broaden the scope to training-free guidance, since constrained sampling can be viewed as a specific instance with stricter requirements.

\textbf{Constrained sampling} methods can be broadly categorized by how strictly they enforce constraints in the formulation. A common approach for \textit{soft-constrained sampling} is to formulate the constraint as a penalty term within a reward function, using its gradient to guide the generation process. Mizuta et al. \cite{mizuta2024cobl} design rewards based on control barrier functions (CBFs) to promote collision avoidance in diffusion-based robot control. Carvalho et al. \cite{carvalho2025motion} construct cost functions that encode safe behaviors for motion planning and bias the sampler accordingly. For applications requiring strict feasibility, \textit{hard-constrained sampling} is necessary. The predominant solution is projection-based methods, where samples are projected onto the feasible set, differing primarily in when the projection is applied. A straightforward approach is \textit{post-hoc projection}. Power et al. \cite{power2023sampling} propose a post-processing step after diffusion sampling to ensure feasibility. Christopher et al. \cite{christopher2024constrained} demonstrate that such post-processing can cause the generated samples to diverge significantly from the original data distribution. To mitigate this, they propose \textit{per-step projection} during Langevin-style diffusion sampling, which is theoretically interpreted as projected gradient ascent on the data log-likelihood. Römer et al. \cite{romer2025diffusion} use per-step projection for diffusion-based planning, along with model predictive control (MPC) to form closed-loop control. Zampini et al. \cite{zampini2025training} integrate per-step projection sampling with Stable Diffusion for text-to-image generation under stringent constraints. Luan et al. \cite{luan2025projected} use per-step projection for coupled generation tasks like creating image pairs. A less restrictive approach is \textit{late-stage projection}, where corrections are applied only in the later steps of the sampling process. Yuan et al. \cite{yuan2023physdiff} apply physics-based projections in late diffusion steps to improve motion plausibility. Bouvier et al. \cite{bouvier2025ddat} use late-stage projections to ensure synthetic robot trajectories are dynamically feasible. \revised{Yang et al. \cite{yang2025safeflowmatcher} propose a two-phase prediction-correction scheme that first generates a coarse sample from the pretrained model and then applies a CBF-based correction on a vanishing time-scaled flow dynamics to steer it toward the safe set.}
Another less restrictive strategy is \textit{per-step projection with early-stage relaxation}, where constraints are enforced in a progressively tightening manner, being relaxed during the early sampling steps and imposed strictly in the later stages. Zhang et al. \cite{zhang2025constrained} enforce constraints in the sampling process via a Lagrangian formulation. As multipliers start small, the constraints are effectively relaxed early on and become progressively tighter as the multipliers increase. Xiao et al. \cite{xiao2025safediffuser} incorporate a CBF into diffusion sampling for safe planning. The CBF condition permits violations before the sample enters the feasible set, thereby relaxing constraints in the early stages.
\revised{To provide a clear taxonomy, we summarize the aforementioned methods in Table~\ref{tab:constrained_sampling_taxonomy}, categorizing them by constraint type, cost handling, core technique, and application domain.}

\textbf{Optimal control taxonomy.} Optimal control provides a framework for computing control policies that optimize a performance criterion while respecting system dynamics and constraints \cite{kirk2004optimal}. We distinguish \textit{global} and \textit{local} formulations. A global formulation seeks an optimal closed-loop policy for every state in the admissible state space and is classically obtained by solving the Hamilton-Jacobi-Bellman (HJB) equation. A local formulation fixes an initial state and computes an optimal state-control trajectory, yielding an open-loop solution. For local problems, there are two primary families of methods: \textit{indirect} and \textit{direct} \cite{von1992direct}. Indirect methods follow an ``optimize-then-discretize'' approach. They first derive the first-order necessary conditions for the continuous-time system, typically using Pontryagin's Maximum Principle (PMP). These conditions are then discretized and solved, often as a two-point boundary-value problem or with a gradient-based algorithm. However, state constraints make PMP notoriously hard to apply, as they create jumps in the costate at entry and exit times, introduce constrained arcs, and can cause singularity issues \cite{bell1975singular}. Direct methods, in contrast, use a ``discretize-then-optimize'' strategy. They transcribe the continuous-time problem into a finite-dimensional nonlinear optimization problem using techniques like shooting or collocation. The resulting optimization problem is then solved numerically. This formulation naturally accommodates constraints, making direct methods well-suited for such problems. An orthogonal classification is \textit{deterministic} versus \textit{stochastic} optimal control, which depends on whether the system dynamics contain stochastic elements \cite{fleming2012deterministic}. In robotics, \textit{trajectory optimization} is the preferred term for local optimal control and, because constraints are pervasive, it often implicitly signals a direct-method formulation, hence our title.

\textbf{Optimal control in generative modeling.} A growing line of work develops an optimal control viewpoint on sampling from unnormalized distributions \cite{tzen2019theoretical, zhang2022path, vargas2023denoising, berner2024an, havens2025adjoint}. For fine-tuning diffusion models, Domingo-Enrich et al. \cite{domingo2025adjoint} propose an adjoint matching framework from indirect stochastic optimal control. They start from PMP and derive a refined loss that exhibits better convergence and stability; Zhao et al. \cite{zhao2025score} take a global optimal control perspective by approximately solving the HJB equation. Regarding training-free guidance, Rout et al. \cite{rout2025rb} propose a stylization method for text-to-image diffusion models grounded in PMP, and Wang et al. \cite{wang2025training} derive PMP-based gradient for guiding flow-matching models. In this work, we address the problem of hard-constrained sampling through the lens of direct optimal control, a perspective that has been largely unexplored in generative modeling.

\textbf{Training-free guidance.} There has been extensive research on this topic since seminal works such as \cite{dhariwal2021diffusion, ho2022classifier}. A recurring theme is to exploit posterior information during sampling \cite{chung2023diffusion, bansal2023universal}. We highlight some more recent gradient-based approaches from distinct viewpoints: optimization \cite{guo2024gradient}, variational inference \cite{pandey2025variational}, conditional and marginal paths \cite{feng2025on}, and optimal control \cite{rout2025rb, wang2025training}. Additionally, there is a growing interest in handling non-differentiable objectives \cite{huang2024symbolic, jain2025diffusion, guo2025training}. Our work contributes to this literature by enabling not only reward optimization but also enforcement of hard constraints at inference time.

\begin{table*}[htbp]
\ifmarked\color{blue}\fi
\centering
\caption{\revised{Comparison of constrained sampling methods. ``Constraint'' indicates whether a method handles constraints as soft penalties or as explicit hard requirements. ``Cost'' indicates whether the method explicitly and jointly handles an auxiliary cost together with constraint enforcement. ``Technique'' indicates the primary method used for constraint enforcement, according to the categorization outlined in the Related Work section. ``Major Application'' refers to the primary evaluation domain, and methods evaluated across multiple domains are labeled as ``General''.}}
\begin{tabular}{l c c c c}
\toprule
\textbf{Method} & \textbf{Constraint} & \textbf{Cost} & \textbf{Technique} & \textbf{Major Application} \\
\midrule

Mizuta et al.~\cite{mizuta2024cobl}
& Soft
& Yes
& Soft guidance via CBF-based rewards
& Robotics \\

Carvalho et al.~\cite{carvalho2025motion}
& Soft
& Yes
& Soft guidance via violation penalties
& Robotics \\

Power et al.~\cite{power2023sampling}
& Hard
& No
& Post-hoc projection
& Robotics \\

Christopher et al.~\cite{christopher2024constrained}
& Hard
& No
& Per-step projection
& General \\

R{\"o}mer et al.~\cite{romer2025diffusion}
& Hard
& No
& Per-step projection
& Robotics \\

Zampini et al.~\cite{zampini2025training}
& Hard
& No
& Per-step projection
& Vision \\

Luan et al.~\cite{luan2025projected}
& Hard
& No
& Per-step projection
& General \\

Yuan et al.~\cite{yuan2023physdiff}
& Hard
& No
& Late-stage projection
& Graphics \\

Bouvier et al.~\cite{bouvier2025ddat}
& Hard
& No
& Late-stage projection
& Robotics \\

Yang et al.~\cite{yang2025safeflowmatcher}
& Hard
& No
& Late-stage projection
& Robotics \\

Zhang et al.~\cite{zhang2025constrained}
& Hard
& No
& Per-step projection with early-stage relaxation
& Robotics \\

Xiao et al.~\cite{xiao2025safediffuser}
& Hard
& No
& Per-step projection with early-stage relaxation
& Robotics \\

\algname\ (Ours)
& Hard
& Yes
& Trajectory optimization
& General \\

\bottomrule
\end{tabular}
\label{tab:constrained_sampling_taxonomy}
\end{table*}

\section{Background}

Consider a time-indexed family of random variables $X_t \in \mathbb{R}^d$ with density $p_t$ for $t \in [0,1]$. Let $v:[0,1]\times\mathbb{R}^d \to \mathbb{R}^d$ denote the marginal velocity field. The associated flow map $\Phi:[0,1]\times[0,1]\times\mathbb{R}^d \to \mathbb{R}^d$ is defined as the solution of an ordinary differential equation (ODE),
\begin{equation}
\nonumber
\left\{\begin{array}{l}
\frac{d}{d \tau} \Phi_{s \rightarrow \tau}(x)=v_\tau\left(\Phi_{s \rightarrow \tau}(x)\right) \\
\Phi_{s \rightarrow s}(x)=x,
\end{array}\right.
\end{equation}
where $\tau \in [s,t]$.
The pushforward of initial density $p_0$ by the flow yields the marginal probability path $p_t=(\Phi_{0\rightarrow t})_{\sharp}p_0$, which satisfies the continuity equation
\begin{equation}
\nonumber
\partial_t p_t+\nabla \cdot \left(p_t v_t\right)=0, \quad p_0.
\end{equation}
We can sample from the terminal density $p_1$ by drawing $x_0\sim p_0$ and integrating the ODE on $[0,1]$:
\begin{equation}
\nonumber
x_1=\Phi_{0\rightarrow 1}(x_0)\sim p_1.
\end{equation}

Flow matching \cite{lipman2023flow} is an efficient, simulation-free method for training a neural velocity field $v_t^\theta(x)$ to approximate $v_t(x)$, given samples from the target distribution $p_1$. The key idea is to design a conditional probability path $\{p_{t\mid Z}\}_{t\in[0,1]}$ that bridges $p_0$ and $p_1$, with $Z$ as the conditioning variable, typically $Z=(X_0,X_1)$. The law of $Z$, denoted $\pi_{0,1}$, is called the coupling. For each realization of $Z$, denote the conditional velocity field as $v_{t\mid Z}(\cdot\mid Z)$. The marginal and conditional velocity fields are related by:
\begin{equation}
\nonumber
v_t(x)
= \mathbb{E}_{Z\sim \pi_{0,1}, X_t\mid Z\sim p_{t\mid Z}}
\left[v_{t\mid Z}(X_t\mid Z)\middle| X_t=x\right].
\end{equation}
The conceptual flow-matching loss compares $v_t^\theta(x)$ to marginal velocity field $v_t(x)$:
\begin{equation}
\nonumber
\mathcal{L}_{\mathrm{FM}}(\theta)
= \mathbb{E}_{t\sim \mathcal{U}[0,1], X_t\sim p_t}
\left[\left\|v_t^\theta(X_t)-v_t(X_t)\right\|_2^2\right],
\end{equation}
which is intractable since $v_t(x)$ is unknown. The conditional flow-matching loss instead compares $v_t^\theta(x)$ to the conditional velocity field $v_{t\mid Z}(x\mid Z)$:
\begin{equation}
\nonumber
\begin{aligned}
\mathcal{L}_{\mathrm{CFM}}(\theta)
&= \mathbb{E}_{\substack{t\sim \mathcal{U}[0,1],\; Z\sim \pi_{0,1},\\ X_t\mid Z\sim p_{t\mid Z}}}
\left[\left\|v_t^\theta(X_t)-v_{t\mid Z}(X_t\mid Z)\right\|_2^2\right].
\end{aligned}
\end{equation}
It is shown that $\nabla_\theta \mathcal{L}_{\mathrm{CFM}}(\theta)=\nabla_\theta \mathcal{L}_{\mathrm{FM}}(\theta)$, thus we can train $v_t^{\theta}(x)$ with $\mathcal{L}_{\mathrm{CFM}}(\theta)$.

To construct the conditional velocity field, a widely adopted choice is to use an affine conditional path
\begin{equation}
\nonumber
X_t\mid Z=\alpha_t X_1+\beta_t X_0,
\end{equation}
with scheduler $(\alpha_t,\beta_t)$ satisfying $\alpha_0=0$, $\alpha_1=1$, $\beta_0=1$, $\beta_1=0$. In this case, $p_{t\mid Z}=\delta_{\alpha_t X_1+\beta_t X_0}$, and the corresponding conditional velocity is
\begin{equation}
\nonumber
v_{t\mid Z}(X_t\mid Z)=\dot{\alpha}_t X_1+\dot{\beta}_t X_0,
\end{equation}
for $X_t=\alpha_t X_1+\beta_t X_0$.

\section{Problem Formulation}

In this paper, we introduce a training-free method for hard-constrained sampling from a pretrained flow-matching model, $v_t^{\theta}(x)$. Our approach is designed to operate with fixed model parameters $\theta$, which enforce constraints by steering the sampling process. We do not consider fine-tuning or inference-time parameter update, which are a separate line of work and typically require substantially more computation. We explain the problem formulation in this section.

Suppose the initial distribution is given by $p_0$. The terminal distribution under the flow map is $\bar{\mu}=(\Phi_{0\to 1}^{\theta})_{\sharp}p_0$, obtained by sampling $x_0\sim p_0$ and integrating $\dot{x}_t=v_t^{\theta}(x_t)$ to $t=1$. We call $\bar{\mu}$ the \textit{nominal distribution}, produced by the uncontrolled sampler (no interference during sampling).

In practice, a nominal distribution may be insufficient, as its samples often fail to meet specific task requirements. Many settings impose \textit{hard constraints}, denoted as \(h(x)\le 0\), which are non-negotiable conditions that all samples must satisfy. For example, in real-world robotics, sampled trajectories must be collision-free, and sampled actions must respect the robot's physical limits. In addition to these strict requirements, we often aim to minimize \textit{costs}, \(C(x)\), such as the energy consumption of a robot trajectory. These costs are secondary. Although lower values are desirable, higher costs are acceptable when necessary to satisfy the primary hard constraints. Moreover, we prefer that the adjusted distribution stay close to the nominal one \(\bar{\mu}\).

\begin{remark}
    \label{constraints_form}
    We use the shorthand \(h(x)\le 0\) to represent multiple constraints. In particular, if \(h(x)=(h_1(x),\ldots,h_m(x))\), the inequality is understood componentwise: \(h_i(x)\le 0\) for all \(i=1,\ldots,m\). Equality constraints \(g(x)=0\) are handled by two inequalities \(g(x)\le 0\) and \(-g(x)\le 0\).
\end{remark}

Concretely, given \(\bar{x}_0\sim p_0\), we seek a per-sample refinement procedure that steers each trajectory during ODE integration to a terminal state \(\tilde{x}_1\) such that (i) the hard constraints \(h(\tilde{x}_1)\le 0\) are satisfied, (ii) the cost \(C(\tilde{x}_1)\) is typically low, and (iii) collectively, over draws of \(\bar{x}_0\), the resulting terminal distribution \(\tilde{x}_1\sim\tilde{\mu}\) remains close to the nominal distribution \(\bar{\mu}\).

This ODE-steering perspective naturally fits within local optimal control. We introduce a control input \(u_t\) to perturb the learned velocity field \(v_t^{\theta}(x)\) during sampling. The continuous-time optimal control problem is presented as follows.

\begin{problem}[Continuous-time formulation]
\label{continuous_oc_problem}\mbox{}\\
Given initial state \(\bar{x}_0\sim p_0\), solve the following problem to obtain $x_1$ as the sample.
\begin{equation}
\label{continuous_oc_formulation}
\begin{array}{@{}l @{\quad} l@{}}
\underset{\{x_t,u_t\}_{t\in[0,1]}}{\min} &
\begin{array}[t]{@{}l@{}}
C(x_1) + \lambda_{\textup{oc}} \int_{0}^{1} \tfrac{1}{2}\lVert u_t\rVert_2^2  dt \\
\begin{array}[t]{@{}ll@{}}
\textup{s.t.} & x_0 = \bar{x}_0, \\
            & \dot{x}_t = v_t^{\theta}(x_t) + u_t, \\
            & h(x_{1}) \le 0.
\end{array}
\end{array}
\end{array}
\end{equation}
\end{problem}

The constraints in \eqref{continuous_oc_formulation} have three components. The initial condition \(x_0=\bar{x}_0\) ensures that the controlled trajectory starts from the same point as the uncontrolled one, which is drawn from the initial distribution \(p_0\). The dynamics \(\dot{x}_t=v_t^{\theta}(x_t)+u_t\) describe how the state evolves under the influence of both the neural velocity field and the control input. The terminal constraint \(h(x_1)\le 0\) enforces that the final state, which is the true sample of interest, satisfies the hard constraint. The objective function \(J\) consists of two terms: the terminal cost \(C(x_1)\), which encourages desirable attributes in the sample, and an integral cost \(\frac{\lambda_{\mathrm{oc}}}{2}\int_0^1 \lVert u_t\rVert_2^2 dt\), which penalizes large control efforts to keep the controlled trajectory close to the nominal one.

The following sections will delve into both algorithmic and theoretical aspects. Algorithmically, solving \eqref{continuous_oc_formulation} online for each sample is extremely challenging: the neural dynamics \(v_t^{\theta}\) is highly nonlinear, and the terminal constraint \(h(x_1)\le 0\), coupled with the dynamics, makes the feasible sets for $u_t$ and $x_t$ (for $t<1$) even more complex. We therefore seek principled transformations of Problem~\ref{continuous_oc_problem} into a tractable surrogate that can be solved efficiently at sampling time while retaining the essential structure of the original formulation. Theoretically, we quantify the approximation errors incurred by our surrogate, so that the gap between the surrogate and the original formulation \eqref{continuous_oc_formulation} is explicit and interpretable. These analyses clarify what the proposed algorithm optimizes and how it balances computational efficiency and formulation optimality.

\section{Algorithm Design}

In this section, we develop a scalable algorithm to solve the optimal control problem \eqref{continuous_oc_formulation} efficiently and effectively. Fig.~\ref{fig:roadmap} illustrates the roadmap of problem transformations, which we explain in detail below.

\begin{figure*}[htbp]
    \centering
    \includegraphics[width=\textwidth]{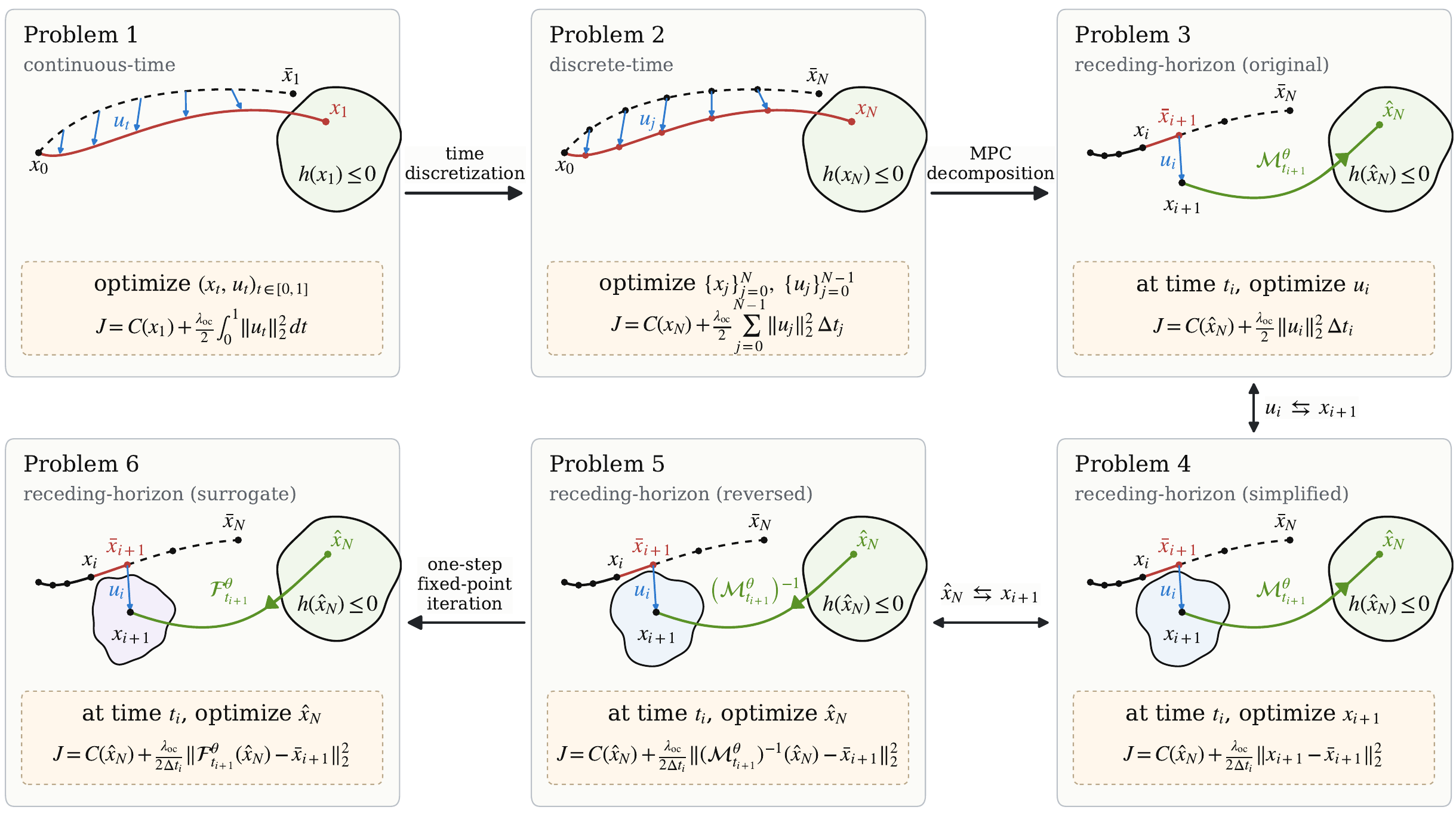}
    \caption{\revised{Visual roadmap for the derivation of HardFlow. Between successive problems, single arrows indicate approximations, while double arrows indicate equivalent reparameterizations.}}
    \label{fig:roadmap}
\end{figure*}

We begin by clarifying the scope. Problem~\ref{continuous_oc_problem} is a local formulation: for a fixed initial condition \(x_0=\bar{x}_0\), we seek an open-loop control \(u_t\) on \(t\in[0,1]\). This stands in contrast to a global formulation that would compute an optimal feedback policy \(u_t(x)\) for all states. The global problem is governed by the Hamilton-Jacobi-Bellman (HJB) equation and suffers from the curse of dimensionality. Data-driven approximations for HJB, such as reinforcement learning, require additional training and fall outside the scope of inference-time guidance.

Recent training-free guidance methods \cite{rout2025rb, wang2025training} for diffusion and flow-matching models have followed the indirect (``optimize-then-discretize'') route for local optimal control, making use of Pontryagin's Maximum Principle (PMP). However, as discussed in Related Work, indirect methods are notoriously difficult to apply to problems with state constraints, which is the core challenge we aim to address. This limitation motivates our adoption of a direct approach. We discretize the continuous-time dynamics and transcribe \eqref{continuous_oc_formulation} into a finite-dimensional constrained optimization problem, about which we can reason comprehensively.

\revised{We discretize the continuous-time dynamics using the forward Euler method. This choice is made for notational simplicity without loss of generality, as the framework presented here extends to higher-order integrators \cite{karras2022elucidating}. It is also consistent with standard sampling practices in flow-matching models \cite{esser2024scaling, polyak2024movie, domingo2025adjoint}.} Define $0=t_0<t_1<\cdots<t_N=1$, and $\Delta t_j=t_{j+1}-t_j$ for $j=0,1,\cdots,N-1$. We arrive at the following problem.

\begin{problem}[Discrete-time formulation]
\label{discrete_oc_problem}\mbox{}\\
Given initial state \(\bar{x}_0\sim p_0\), solve the following problem to obtain $x_N$ as the sample.
\begin{equation}
\label{discretized_oc_formulation}
\begin{array}{@{}l @{\quad} l@{}}
\underset{\left\{x_j\right\}_{j=0}^N, \left\{u_j\right\}_{j=0}^{N-1}}{\min} &
\begin{array}[t]{@{}l@{}}
C(x_N) + \lambda_{\textup{oc}} \sum_{j=0}^{N-1} \frac{1}{2}\left\|u_j\right\|_2^2 \Delta t_j \\
\begin{array}[t]{@{}ll@{}}
\textup{s.t.} & x_0 = \bar{x}_0, \\
            & x_{j+1}=x_j+v_{t_j}^\theta\left(x_j\right) \Delta t_j+u_j \Delta t_j,\\&j=0,1, \cdots, N-1, \\
            & h(x_{N}) \le 0.
\end{array}
\end{array}
\end{array}
\end{equation}
\end{problem}

The decision variables in Problem~\ref{discrete_oc_problem} are the discrete states and controls, for a total of \(2Nd+d\) scalars. Solving this problem directly is computationally demanding, especially for fine-grained discretizations (large \(N\)) or high-dimensional data such as images (large \(d\)). These challenges are compounded by the complex, multi-step coupling between states and controls via highly nonconvex neural network dynamics. To make the problem tractable, we employ a model predictive control (MPC) framework. This strategy decomposes the long-horizon optimization into a sequence of more manageable, short-horizon subproblems. The long-horizon coupling in Problem~\ref{discrete_oc_problem} arises from the terminal cost $C(x_N)$ and terminal constraint $h(x_N)\le 0$. Therefore, the key to applying MPC is to construct effective proxies for $C(x_N)$ and $h(x_N)\le 0$ at an intermediate state $x_i$. Fortunately, diffusion and flow-matching models possess a unique structure that facilitates this approximation. Specifically, they provide a posterior estimate of the terminal state $x_N$ given an intermediate state $x_i$. For affine conditional paths, the following results hold.

\begin{lemma}[Posterior mean \cite{feng2025on}]
\label{posterior_mean_lemma}
Let $Z=(X_0,X_1)\sim\pi_{0,1}$. Consider the affine conditional path
\(
X_t\mid Z=\alpha_t X_1+\beta_t X_0
\) for $t\in[0,1]$ with differentiable scheduler $(\alpha_t,\beta_t)$.
The conditional velocity field is therefore
\(v_{t\mid Z}(X_t\mid Z)=\dot{\alpha}_t X_1+\dot{\beta}_t X_0\).
Suppose $v_t(x)$ is the marginal velocity field, i.e., $v_t(x)=\mathbb{E}\left[v_{t\mid Z}(X_t\mid Z)\middle| X_t=x\right]$.
Define
\(
\Lambda_t \coloneqq \alpha_t \dot{\beta}_t-\dot{\alpha}_t \beta_t.
\)
Assume \(\Lambda_t \neq 0\) for all \(t \in [0,1]\).
Then, for $x\in\mathbb{R}^d$, we have
\begin{equation}
\nonumber
\mathcal{M}_t(x)\coloneqq \mathbb{E}\left[X_1 \mid X_t=x\right]=
\frac{\dot{\beta}_t x-\beta_t v_t(x)}{\Lambda_t},
\end{equation}
\begin{equation}
\nonumber
\mathcal{W}_t(x)\coloneqq \mathbb{E}\left[X_0 \mid X_t=x \right]=
\frac{-\dot{\alpha}_t x+\alpha_t v_t(x)}{\Lambda_t}.
\end{equation}
Additionally, the following identity holds:
\begin{equation}
\label{posterior_identity}
x=\alpha_t \mathcal{M}_t(x)+\beta_t \mathcal{W}_t(x).
\end{equation}
\end{lemma}

With Lemma~\ref{posterior_mean_lemma}, we can approximate the terminal state \(x_N\) given an intermediate state \(x_i\) by
\begin{equation}
\nonumber
\mathcal{M}_{t_i}^{\theta}(x_i)\coloneqq \frac{\dot{\beta}_{t_i} x_i-\beta_{t_i} v_{t_i}^\theta\left(x_i\right)}{\Lambda_{t_i}}, 
\end{equation}
for $i=0,1,\cdots,N$. In particular, $\mathcal{M}_{t_N}^{\theta}(x_N)=x_N$ under standard boundary conditions $\alpha_{0}=\beta_{1}=0$ and $\alpha_{1}=\beta_{0}=1$.

We now apply MPC to decompose Problem~\ref{discrete_oc_problem} into a sequence of one-step subproblems, yielding the following receding-horizon scheme.

\begin{problem}[Receding-horizon formulation, original]
\label{receding_horizon_problem_original}\mbox{}\\
Given initial state \(\bar{x}_0\sim p_0\), initialize $x_0=\bar{x}_0$ and carry out the following recursion for $i=0,1,\cdots,N-1$ to obtain $x_N$ as the sample.
\begin{equation}
\label{receding_formulation_original}
\left\{\begin{array}{l}
\begin{array}{@{}l @{\quad} l@{}}
u_i^*=\underset{u_{i}}{\operatorname{argmin}} &
\begin{array}[t]{@{}l@{}}
C(\widehat{x}_N) + \frac{\lambda_{\textup{oc}}}{2}\left\|u_i\right\|_2^2 \Delta t_i \\
\begin{array}[t]{@{}ll@{}}
\textup{s.t.}
            & x_{i+1}=x_i+v_{t_i}^\theta\left(x_i\right) \Delta t_i+u_i \Delta t_i,\\
            & h(\widehat{x}_{N}) \le 0, \\
            & \widehat{x}_{N} = \mathcal{M}_{t_{i+1}}^{\theta}(x_{i+1}).
\end{array}
\end{array}
\end{array} \\
x_{i+1} =x_i+v_{t_{i}}^\theta\left(x_{i}\right)\Delta t_{i} +u_{i}^* \Delta t_{i}\\
\end{array}\right.
\end{equation}
\end{problem}

Problem~\ref{receding_horizon_problem_original} computes controls \(u_i\) step by step. Intuitively, \(u_i\) is chosen so that after it is applied to reach \(x_{i+1}\), the resulting terminal state will be both low-cost and feasible, assuming no future controls are used. A direct approach would be to simulate the remaining trajectory from \(x_{i+1}\) to \(x_N\) at every step and then evaluate \(C(x_N)\) and \(h(x_N)\), which is computationally expensive. Instead, we use the posterior estimate \(\widehat{x}_N=\mathcal{M}_{t_{i+1}}^{\theta}(x_{i+1})\) to approximate the terminal state. This captures the effect of \(u_i\) on the terminal outcome at an efficient \(\mathcal{O}(1)\) cost per step. The control regularization \(\tfrac{\lambda_{\mathrm{oc}}}{2}\|u_i\|_2^2\Delta t_i\) discourages large deviations from the nominal dynamics.

From this perspective, Problem~\ref{receding_horizon_problem_original} is a principled approximation of Problem~\ref{discrete_oc_problem}: it preserves the core structure while greatly improving tractability. The approximation errors arise from two sources: (i) the posterior estimate \(\widehat{x}_N\) is not exact, and (ii) future controls $\left\{u_j\right\}_{j=i+1}^{N-1}$ are ignored when optimizing \(u_i\). We will rigorously quantify these errors in the next section.

Observe that, since the control enters the dynamics additively, \(u_i\) can be viewed as a perturbation to the nominal next state \(\bar{x}_{i+1}=x_i+v_{t_i}^\theta(x_i)\Delta t_i\). This motivates a change of variables from \(u_i\) to \(x_{i+1}\), leading to an equivalent formulation as Problem~\ref{receding_horizon_problem_simplified}. This transformation is beneficial as it clarifies the structure and enables further manipulations.

\begin{problem}[Receding-horizon formulation, simplified]
\label{receding_horizon_problem_simplified}\mbox{}\\
Given initial state \(\bar{x}_0\sim p_0\), initialize $x_0=\bar{x}_0$ and carry out the following recursion for $i=0,1,\cdots,N-1$ to obtain $x_N$ as the sample.
\begin{equation}
\left\{\begin{array}{l}
\bar{x}_{i+1} = x_i + v_{t_i}^\theta(x_i) \Delta t_i \\
\begin{array}{@{}l @{\quad} l@{}}
x_{i+1}^*=\underset{x_{i+1}}{\operatorname{argmin}} &
\begin{array}[t]{@{}l@{}}
C(\widehat{x}_N) + \frac{\lambda_{\textup{oc}}}{2\Delta{t_{i}}}\left\|x_{i+1}-\bar{x}_{i+1}\right\|_2^2 \\
\begin{array}[t]{@{}ll@{}}
\textup{s.t.}
            & h(\widehat{x}_{N}) \le 0, \\
            & \widehat{x}_{N} = \mathcal{M}_{t_{i+1}}^{\theta}(x_{i+1}).
\end{array}
\end{array}
\end{array} \\
x_{i+1} =x_{i+1}^*\\
\end{array}\right.
\end{equation}
\end{problem}

Problem~\ref{receding_horizon_problem_simplified} involves two coupled states: the decision variable \(x_{i+1}\) and the predicted terminal state \(\widehat{x}_N\). While the optimization is over \(x_{i+1}\), the terminal cost and constraints are applied to \(\widehat{x}_N\). The two states are linked by the mapping $\widehat{x}_N=\mathcal{M}_{t_{i+1}}^{\theta}(x_{i+1})$, which is highly nonlinear due to the embedded neural network \(v_{t_{i+1}}^{\theta}(\cdot)\). Consequently, the feasible set for \(x_{i+1}\), implicitly defined as $\left\{x_{i+1} \mid h\left(\mathcal{M}_{t_{i+1}}^\theta\left(x_{i+1}\right)\right) \leq 0\right\}$, is significantly complicated. This structure poses a significant challenge for numerical solvers, as it can be difficult to find even a single feasible point, regardless of the simplicity of \(h(\cdot)\) (e.g., linear).

Since feasibility is of paramount importance, we aim to work directly with the original constraint \(h(x_N)\le 0\), avoiding the complex distortion introduced by the neural network \(v_{t_{i+1}}^{\theta}(\cdot)\). To this end, we propose a reversed formulation. We treat \(\widehat{x}_N\) as the new decision variable and then express \(x_{i+1}\) in terms of \(\widehat{x}_N\). Assuming the mapping \(\mathcal{M}_{t_{i+1}}^{\theta}(\cdot)\) is invertible, we have the following formulation.

\begin{problem}[Receding-horizon formulation, reversed]
\label{receding_horizon_problem_reversed}\mbox{}\\
Given initial state \(\bar{x}_0\sim p_0\), initialize $x_0=\bar{x}_0$ and carry out the following recursion for $i=0,1,\cdots,N-1$ to obtain $x_N$ as the sample.
\begin{equation}
\resizebox{1.025\displaywidth}{!}{$
\left\{
\begin{array}{l}
\bar{x}_{i+1}=x_i+v_{t_i}^\theta\left(x_i\right)\Delta t_i\\
\begin{array}{@{}l @{\quad} l@{}}
\widehat{x}_N^{*}=\underset{\widehat{x}_N}{\operatorname{argmin}} &
\begin{array}[t]{@{}l@{}}
C(\widehat{x}_N)+\frac{\lambda_{\textup{oc}}}{2\Delta t_i}\left\|(\mathcal{M}_{t_{i+1}}^{\theta})^{-1}(\widehat{x}_{N})-\bar{x}_{i+1}\right\|_2^2\\
\begin{array}[t]{@{}ll@{}}
\textup{s.t.} & h(\widehat{x}_{N}) \le 0.
\end{array}
\end{array}
\end{array}\\
x_{i+1}=(\mathcal{M}_{t_{i+1}}^{\theta})^{-1}(\widehat{x}_{N}^{*})\\
\end{array}
\right.
$}
\end{equation}
\end{problem}

To make Problem~\ref{receding_horizon_problem_reversed} practical, we need to characterize the inverse mapping \((\mathcal{M}_{t_{i+1}}^{\theta})^{-1}(\cdot)\). From \eqref{posterior_identity} in Lemma~\ref{posterior_mean_lemma}, for any \(x\) we have
\begin{equation}
\nonumber
x=\alpha_{t_{i+1}} \mathcal{M}_{t_{i+1}}^{\theta}(x)+\beta_{t_{i+1}} \mathcal{W}_{t_{i+1}}^{\theta}(x),
\end{equation}
where
\begin{equation}
\nonumber
\mathcal{W}_{t_{i+1}}^{\theta}(x)=\frac{-\dot{\alpha}_{t_{i+1}} x+\alpha_{t_{i+1}} v_{t_{i+1}}^\theta(x)}{\Lambda_{t_{i+1}}}.
\end{equation}
Then $y=\mathcal{M}_{t_{i+1}}^{\theta}(x)$ is equivalent to $x=\alpha_{t_{i+1}} y+\beta_{t_{i+1}} \mathcal{W}_{t_{i+1}}^{\theta}(x)$ for any $y$. Consequently, inverting $y=\mathcal{M}_{t_{i+1}}^{\theta}(x)$ amounts to finding a fixed point of the mapping
\begin{equation}
\label{fixed_point_operator}
T^y_{t_{i+1}}(x)=\alpha_{t_{i+1}} y+\beta_{t_{i+1}} \mathcal{W}_{t_{i+1}}^{\theta}(x),
\end{equation}
i.e., solving \(x^*=T^y_{t_{i+1}}(x^*)\). Assuming \(T^y_{t_{i+1}}(\cdot)\) is a contraction, the fixed point from any initial guess \(x^{(0)}\) is the limit of the iterations
\begin{equation}
\label{fixed_point_inverse}
(\mathcal{M}_{t_{i+1}}^{\theta})^{-1}(y)=\lim_{k\to\infty} \left(T^y_{t_{i+1}}\right)^k(x^{(0)}).
\end{equation}
In our setting, a natural initial guess is the nominal next state \(\bar{x}_{i+1}=x_i+v_{t_i}^\theta\left(x_i\right) \Delta t_i\). Truncating the iterations after one step yields
\begin{equation}
\nonumber
(\mathcal{M}_{t_{i+1}}^{\theta})^{-1}(y)\approx T^y_{t_{i+1}}(\bar{x}_{i+1})=\alpha_{t_{i+1}} y+\beta_{t_{i+1}} \mathcal{W}_{t_{i+1}}^{\theta}\left(\bar{x}_{i+1}\right).
\end{equation}
This one-step cut-off is a deliberate design choice, presenting a linear approximation of $(\mathcal{M}_{t_{i+1}}^{\theta})^{-1}(y)$ that requires only a single forward pass of \(v_{t_{i+1}}^{\theta}(\cdot)\) at \(\bar{x}_{i+1}\). This eliminates the need to evaluate input gradients of \(v_{t_{i+1}}^{\theta}(\cdot)\) within the associated optimization, greatly improving computational efficiency. The approximation error will be analyzed in the Theoretical Analysis section. We thus arrive at the final surrogate formulation as Problem~\ref{receding_horizon_problem_surrogate}.

\begin{problem}[Receding-horizon formulation, surrogate]
\label{receding_horizon_problem_surrogate}\mbox{}\\
Given initial state \(\bar{x}_0\sim p_0\), initialize $x_0=\bar{x}_0$ and carry out the following recursion for $i=0,1,\cdots,N-1$ to obtain $x_N$ as the sample.
\begin{equation}
\label{receding_oc_formulation_surrogate}
\left\{\begin{array}{l}
\bar{x}_{i+1}=x_i+v_{t_i}^\theta\left(x_i\right) \Delta t_i \\
\begin{array}{@{}l @{\quad} l@{}}
\widehat{x}_N^{*}=\underset{\widehat{x}_N}{\operatorname{argmin}} &
\begin{array}[t]{@{}l@{}}
C(\widehat{x}_N) + \frac{\lambda_{\textup{oc}}}{2\Delta{t_{i}}}\left\| \mathcal{F}_{t_{i+1}}^{\theta}(\widehat{x}_{N})-\bar{x}_{i+1}\right\|_2^2 \\
\begin{array}[t]{@{}ll@{}}
\textup{s.t.}
            & h(\widehat{x}_{N}) \le 0.\\
\end{array}
\end{array}
\end{array} \\
x_{i+1} =\mathcal{F}_{t_{i+1}}^{\theta}(\widehat{x}_{N}^{*})\\
\end{array}\right.
\end{equation}
Specifically, $\mathcal{F}_{t_{i+1}}^{\theta}: \mathbb{R}^{d} \to \mathbb{R}^{d}$ is defined as
\begin{equation}
\mathcal{F}_{t_{i+1}}^{\theta}(\widehat{x}_{N})=\alpha_{t_{i+1}} \widehat{x}_{N}+\beta_{t_{i+1}} \mathcal{W}_{t_{i+1}}^{\theta}\left(\bar{x}_{i+1}\right).
\end{equation}
\end{problem}

Since \(\mathcal{F}_{t_{i+1}}^{\theta}(\cdot)\) is affine, we can further simplify the expressions. Evaluating the identity \eqref{posterior_identity} at \(\bar{x}_{i+1}\) gives \(\bar{x}_{i+1}=\alpha_{t_{i+1}} \mathcal{M}_{t_{i+1}}^{\theta}(\bar{x}_{i+1})+\beta_{t_{i+1}} \mathcal{W}_{t_{i+1}}^{\theta}\left(\bar{x}_{i+1}\right)\). Define \(\bar{x}_N\coloneqq\mathcal{M}_{t_{i+1}}^{\theta}(\bar{x}_{i+1})\). Then, for any \(\widehat{x}_N\),
\begin{equation}
\nonumber
\begin{aligned}
\left\|x_{i+1}-\bar{x}_{i+1}\right\|_2^2  & = \left\|\mathcal{F}_{t_{i+1}}^{\theta}(\widehat{x}_N)-\bar{x}_{i+1}\right\|_2^2\\
& = \left\|\alpha_{t_{i+1}} \widehat{x}_{N}-\alpha_{t_{i+1}} \mathcal{M}_{t_{i+1}}^{\theta}(\bar{x}_{i+1})\right\|_2^2 \\
& = \alpha_{t_{i+1}}^2 \left\|\widehat{x}_N-\bar{x}_N\right\|_2^2.
\end{aligned}
\end{equation}
Therefore, the optimization in \eqref{receding_oc_formulation_surrogate} can be expressed entirely in terms of \(\widehat{x}_N\), without involving \(x_{i+1}\). We now summarize the complete workflow in Algorithm~\ref{main_algorithm}, which is named \algname.

\begin{algorithm}[htbp]
    \caption{\algname: Hard-constrained sampling for flow-matching models}
    \label{main_algorithm}
    \KwIn{Initial distribution \(p_0\), learned flow-matching model \(v_t^{\theta}(\cdot)\), cost \(C(\cdot)\), constraint \(h(\cdot)\le 0\), regularization parameter \(\lambda_{\mathrm{oc}}>0\), discretization steps \(N\), time grid \(0=t_0<t_1<\cdots<t_N=1\), and differentiable affine scheduler \((\alpha_t,\beta_t)\).}
    Draw initial state \(\bar{x}_0\sim p_0\) and set \(x_0=\bar{x}_0\)\;
        \For{\(i=0\) \KwTo \(N-1\)}{
            Compute \(\Delta t_i=t_{i+1}-t_i\).

            Compute \(\bar{x}_{i+1}=x_i+v_{t_i}^\theta\left(x_i\right) \Delta t_i\).

            Compute \(\bar{x}_N=\frac{\dot{\beta}_{t_{i+1}} \bar{x}_{i+1}-\beta_{t_{i+1}} v_{t_{i+1}}^\theta\left(\bar{x}_{i+1}\right)}{\alpha_{t_{i+1}} \dot{\beta}_{t_{i+1}}-\dot{\alpha}_{t_{i+1}} \beta_{t_{i+1}}}\).

            Solve the optimization problem
            \begin{equation}
                \label{final_optimization_formulation}
                \begin{array}{@{}l @{\quad} l@{}}
                \widehat{x}_N^{*}=\underset{\widehat{x}_N}{\operatorname{argmin}} &
                \begin{array}[t]{@{}l@{}}
                C(\widehat{x}_N) + \frac{\lambda_{\textup{oc}}}{2\Delta{t_{i}}}\alpha_{t_{i+1}}^2\left\|\widehat{x}_N-\bar{x}_N\right\|_2^2 \\
                \begin{array}[t]{@{}ll@{}}
                \textup{s.t.}
                            & h(\widehat{x}_{N}) \le 0.
                \end{array}
                \end{array}
                \end{array}
            \end{equation}

            Compute \(x_{i+1} = \alpha_{t_{i+1}} \widehat{x}_{N}^{*}+\beta_{t_{i+1}} \frac{-\dot{\alpha}_{t_{i+1}} \bar{x}_{i+1}+\alpha_{t_{i+1}} v_{t_{i+1}}^\theta(\bar{x}_{i+1})}{\alpha_{t_{i+1}} \dot{\beta}_{t_{i+1}}-\dot{\alpha}_{t_{i+1}} \beta_{t_{i+1}}}\). \label{compute_next_state}
        }
    \KwOut{Sample \(x_N\).}
\end{algorithm}

\algname\ is essentially equivalent to solving Problem~\ref{receding_horizon_problem_surrogate}, where the optimization \eqref{final_optimization_formulation} is written in an explicit $\widehat{x}_N$-only form. For completeness, we also expand intermediate expressions and present the final form.

It is worth emphasizing that, although we perform a series of transformations on Problem~\ref{discrete_oc_problem} to obtain Problem~\ref{receding_horizon_problem_surrogate}, there is no relaxation of terminal feasibility. More precisely, we have the following result.

\begin{proposition}
\label{safety_guarantee}
Suppose the feasible set \(\mathcal{S} = \{x \mid h(x) \le 0\}\) is nonempty. The output sample \(x_N\) of Algorithm~\ref{main_algorithm} satisfies the hard constraints \(h(x_N) \le 0\).
\end{proposition}

\revised{Proof of Proposition~\ref{safety_guarantee} is deferred to Appendix~\ref{app:theory_proofs}.}

\begin{remark}
In the \textit{unconstrained} setting, where only a cost \(C(\cdot)\) (equivalently, a reward \(R(\cdot)=-C(\cdot)\)) is present, the task reduces to guided sampling with a reward model, for which multiple formulations and algorithms exist in isolation but have not been unified. Problem~\ref{continuous_oc_problem} without constraints matches the optimal control formulation in \cite{wang2025training, domingo2025adjoint}. Solving Problem~\ref{receding_horizon_problem_simplified} without constraints by gradient descent is akin to the gradient-based methods in \cite{guo2024gradient, pandey2025variational, feng2025on}, which differentiate the predicted state with respect to the current state. Problem~\ref{receding_horizon_problem_surrogate} without constraints is in the same spirit as \cite{wang2023zeroshot} and the proxy formulation in \cite{rout2025rb}, both of which employ DDIM-like updates that jump to the terminal state, optimize it, and map back to the intermediate state, whereas our flow-matching approach is more general since it does not require a Gaussian source.

To the best of our knowledge, the hierarchy from Problem~\ref{continuous_oc_problem} to Problem~\ref{receding_horizon_problem_surrogate} proposed in this paper is the first to systematically connect these formulations through the lens of optimal control. Starting from the continuous-time Problem~\ref{continuous_oc_problem}, discretization, MPC decomposition, a control-to-state change of variables, a reverse reparameterization, and a one-step fixed-point iteration link all the formulations in a principled manner. This unified view, together with the theoretical analysis in the next section, may be of independent interest for reward-only scenarios.
\end{remark}

\begin{remark}
    The constrained optimization problem \eqref{final_optimization_formulation} can be solved using any suitable method, depending on the specific downstream application. In general, it may be solved numerically with constrained optimization algorithms, such as sequential quadratic programming or interior-point methods. Moreover, if the cost \(C(\cdot)\) and constraints \(h(\cdot)\) possess particular structures (e.g., a quadratic cost and linear constraints), specialized solvers or even closed-form solutions may be available.
\end{remark}

In summary, this section utilizes trajectory optimization to enforce hard constraints and minimize costs on samples at inference time. Starting from the continuous-time formulation (Problem~\ref{continuous_oc_problem}), we discretize it (Problem~\ref{discrete_oc_problem}) and distill it (Problem~\ref{receding_horizon_problem_original} to~\ref{receding_horizon_problem_surrogate}) into a tractable algorithm via two major transformations: (i) an MPC decomposition that reduces the full-horizon problem to sequential one-step subproblems, and (ii) a reversed formulation with a single-step fixed-point iteration that sidesteps complex feasible sets and improves computational efficiency.
\revised{We also provide an expanded methodological discussion covering aspects that merit further emphasis and elaboration, which, due to space limitations, is deferred to Appendix~\ref{app:methodological_discussion}.}
Next, we will provide comprehensive theoretical justifications for \algname\ in the Theoretical Analysis section and demonstrate the practical effectiveness through extensive evaluations in the Experiments section.

\section{Theoretical Analysis}

In this section, we present a detailed analysis of the proposed algorithm. Leveraging control-theoretic tools, we quantify the approximation errors introduced by the two major transformations used to derive Problem~\ref{receding_horizon_problem_surrogate} from Problem~\ref{discrete_oc_problem}. \revised{All proofs in this section are deferred to Appendix~\ref{app:theory_proofs}.}

First, we analyze the suboptimality introduced by the MPC decomposition, i.e., the gap between Problem~\ref{discrete_oc_problem} and Problem~\ref{receding_horizon_problem_original}. For notation convenience, define the stage utility $l_i(u)\coloneqq \frac{\lambda_{\mathrm{oc}}}{2}\|u\|_2^2 \Delta t_i$, the (Problem~\ref{discrete_oc_problem}) objective $\mathcal{J}\coloneqq C(x_N)+\sum_{j=0}^{N-1} l_j(u_j)$, the one-step transition $\Psi_i^{\theta}(x,u)\coloneqq x+v_{t_i}^\theta(x)\Delta t_i+u\Delta t_i$ with $\Psi_i^{\theta}(x)\coloneqq \Psi_i^{\theta}(x,0)$, and the terminal feasible set $\mathcal{S}=\{x\mid h(x)\le 0\}$.

We assume that, given an initial state \(x_0\), both Problem~\ref{discrete_oc_problem} and Problem~\ref{receding_horizon_problem_original} admit feasible solutions. We further define the per-step feasible sets for Problem~\ref{receding_horizon_problem_original} as $\widehat{\mathcal{S}}_i=\{x\mid h(\mathcal{M}_{t_i}^\theta(x))\le 0\}$, for $i=1,2,\cdots,N$. Since $\mathcal{M}_{t_N}^\theta(x)=x$, we have $\widehat{\mathcal{S}}_N=\mathcal{S}$. Feasibility of Problem~\ref{receding_horizon_problem_original} then implies each $\widehat{\mathcal{S}}_i$ is nonempty. These assumptions on feasibility are reasonable, since each $u_i$ is unconstrained, and sufficiently large controls can be applied as needed to steer the state into the feasible set.

We also impose the following regularity assumptions: the terminal cost \(C(\cdot)\) is \(L_C\)-Lipschitz continuous on \(\mathcal{S}\), and, for each \(i\), the terminal state estimator \(\mathcal{M}_{t_i}^{\theta}(\cdot)\) is \(L_{\mathcal{M}_x,i}\)-Lipschitz continuous on the region of interest.

For the same initial state \(x_0=\bar{x}_0\), assume both Problem~\ref{discrete_oc_problem} and Problem~\ref{receding_horizon_problem_original} are solved to global optimality, and denote their optimal control sequences by \(u^{\textup{P2}}\) and \(u^{\textup{P3}}\), respectively. Since $\mathcal{J}$ is the objective of Problem~\ref{discrete_oc_problem}, by definition we have $0\le\mathcal{J}(x_0,u^{\textup{P3}})-\mathcal{J}(x_0,u^{\textup{P2}})$. It remains to upper-bound this difference.

\begin{remark}
The global optimality assumption does not conflict with the fact that Problem~\ref{discrete_oc_problem} is a local optimal control formulation. As discussed in Related Work, local versus global optimal control refers to the scope of control: open-loop for a fixed initial state, versus a feedback policy defined on all states. This is distinct from local versus global optimality of a single optimization problem. A local optimal control problem can still be solved to global optimality; in that case, the trajectory under the resulting open-loop control coincides with that produced by the optimal feedback policy starting from the same state.
\end{remark}

Define the terminal value function $V_N(x)\coloneqq C(x) + \iota_{\mathcal{S}}(x)$, where \(\iota_{\mathcal{S}}(x)\) denotes the indicator of \(\mathcal{S}\):
\begin{equation}
\nonumber
\iota_{\mathcal{S}}(x)= \begin{cases}0, & x \in \mathcal{S}, \\ +\infty, & x \notin \mathcal{S}.\end{cases}
\end{equation}
For \(i=0,1,\cdots,N-1\), define the optimal cost-to-go as
\begin{equation}
\nonumber
                \begin{array}{@{}l @{\quad} l@{}}
               V_i(x) \coloneqq \underset{\left\{x_j\right\}_{j=i}^{N}, \left\{u_j\right\}_{j=i}^{N-1}}{\operatorname{min}} &
                \begin{array}[t]{@{}l@{}}
                C(x_N) +\lambda_{\textup{oc}} \sum_{j=i}^{N-1} \frac{1}{2}\left\|u_j\right\|_2^2 \Delta t_j \\
                \begin{array}[t]{@{}ll@{}}
                \textup{s.t.}
                & x_i=x, \\
                & x_{j+1}=\Psi_j^{\theta}(x_j,u_j), \\
                & j=i,i+1,\cdots,N-1, \\
                & h(x_{N}) \le 0.
                \end{array}
                \end{array}
                \end{array}
\end{equation}
For any function \(F:\mathbb{R}^d\to \mathbb{R}\cup\{+\infty\}\) and \(i=0,1,\cdots,N-1\), define the Bellman operators
\begin{equation}
\nonumber
\left(T_i F\right)(x):=\min _u\left\{l_i(u)+F\left(\Psi_i^{\theta}(x, u)\right)\right\},
\end{equation}
which can be equivalently written as
\begin{equation}
\label{bellman_operator_equivalent}
\left(T_i F\right)(x):=\min _{y}
\left\{\frac{\lambda_{\textup{oc}}}{2 \Delta t_i}\left\|y-\Psi_i^{\theta}(x)\right\|_2^2+F(y)\right\}.
\end{equation}
This expression uses the same control-to-state change of variables as in Problem~\ref{receding_horizon_problem_simplified}. We have the following proposition.

\begin{proposition}
\label{bellman_recursion}
For \(i=0,1,\cdots,N-1\), the Bellman recursion \(V_i=T_i V_{i+1}\) holds, and we have \(V_0(x_0)=\mathcal{J}(x_0,u^{\textup{P2}})\).
\end{proposition}

Instead of the true terminal value \(V_N(x)\), Problem~\ref{receding_horizon_problem_original} employs proxy terminal values \(\widehat{V}_i(x)\coloneqq C(\mathcal{M}_{t_i}^\theta(x))+\iota_{\mathcal{S}}(\mathcal{M}_{t_i}^\theta(x))\), for \(i=1,2,\cdots,N\).
Given \(x\in \widehat{\mathcal{S}}_i\) and \(i=1,2,\cdots,N-1\), define the Bellman residual
\begin{equation}
\nonumber
r_i(x)\coloneqq(T_i \widehat{V}_{i+1})(x)-\widehat{V}_i(x).
\end{equation}
We explicitly require \(x\in \widehat{\mathcal{S}}_i\) so that \(\widehat{V}_i(x)\) is finite. The term $(T_i \widehat{V}_{i+1})(x)$ is always finite for any \(x\), since we can always select $y\in \widehat{\mathcal{S}}_{i+1}$ in \eqref{bellman_operator_equivalent}.

To control the mismatch between the exact and proxy value functions, we use the following proposition. We assume the minimizer from \(T_i V_{i+1}\) lies in \(\widehat{\mathcal{S}}_{i+1}\). This effectively requires that the proxy terminal constraints in Problem~\ref{receding_horizon_problem_original} do not exclude the optimal successor state. If this assumption is violated, the bound can be refined by adding a term that accounts for the resulting minimizer mismatch. The extension is routine but tedious, so we present the simpler form here.

\begin{proposition}
\label{bellman_operator_nonexpansive}
For any \(i=0,1,\cdots,N-1\), any \(x\in \mathbb{R}^d\), we have
\begin{equation}
\nonumber
(T_i \widehat{V}_{i+1} - T_i V_{i+1})(x) \le \sup_{y\in \widehat{\mathcal{S}}_{i+1}} (\widehat{V}_{i+1}-V_{i+1})(y).
\end{equation}
\end{proposition}

Assume that for each \(i=1,2,\ldots,N-1\) there exists \(\kappa_i \in(0, \infty)\) such that, for every \(x\in \widehat{\mathcal{S}}_i\), there is a feasible control \(u\) with \(\|u\|_2\le \kappa_i\) and \(\Psi_i^{\theta}(x, u)\in \widehat{\mathcal{S}}_{i+1}\). This is a reasonable assumption that rules out pathological cases. Define the one-step consistency error of the terminal state estimator as
\begin{equation}
\nonumber
\varepsilon_i:=\sup _{x \in \widehat{\mathcal{S}}_i}\left\|\mathcal{M}_{t_{i+1}}^\theta\left(\Psi_i^\theta(x)\right)-\mathcal{M}_{t_i}^\theta(x)\right\|_2,
\end{equation}
and
\begin{equation}
\nonumber
\Gamma_i:=\max \left\{\frac{L_C^2 L_{\mathcal{M}_x,i+1}^2}{2 \lambda_{\textup{oc}}}, \frac{\lambda_{\textup{oc}}}{2} \kappa_i^2+L_C L_{\mathcal{M}_x,i+1} \kappa_i\right\}.
\end{equation}
We have the following result.

\begin{theorem}
\label{mpc_suboptimality_theorem}
For any initial state \(x_0=\bar{x}_0\sim p_0\), if both Problem~\ref{discrete_oc_problem} and Problem~\ref{receding_horizon_problem_original} are solved to global optimality, whose solutions are denoted as \(u^{\textup{P2}}\) and \(u^{\textup{P3}}\), then
\begin{equation}
\label{mpc_suboptimality_bound}
0 \leq \mathcal{J}\left(x_0,u^{\textup{P3}}\right)-\mathcal{J}\left(x_0,u^{\textup{P2}}\right) \leq 2 \sum_{i=1}^{N-1}\left(L_C \varepsilon_i+\Gamma_i \Delta t_i\right).
\end{equation}
\end{theorem}

\begin{remark}
The bound in Theorem~\ref{mpc_suboptimality_theorem} directly reflects the two sources of approximation errors when transforming Problem~\ref{discrete_oc_problem} to Problem~\ref{receding_horizon_problem_original}. First, the one-step consistency errors \(\varepsilon_i\) arise from replacing the true terminal state obtained by ODE simulation with the estimator \(\mathcal{M}_{t_i}^\theta(\cdot)\). With a perfect terminal-state oracle, \(\varepsilon_i\) would vanish to zero. Second, when computing $u_i$, we ignore the future controls \(\left\{u_j\right\}_{j=i+1}^{N-1}\), which introduces the term $\frac{L_C^2 L_{\mathcal{M}_x,i+1}^2}{2 \lambda_{\textup{oc}}}$ in \(\Gamma_i\). The remaining term \(\frac{\lambda_{\textup{oc}}}{2} \kappa_i^2+L_C L_{\mathcal{M}_x,i+1} \kappa_i\) accounts for the control effort needed to steer the state into the feasible set when required.
\end{remark}

\begin{remark}
\label{continuous_time_limit}
Since Problem~\ref{discrete_oc_problem} converges to Problem~\ref{continuous_oc_problem} as $N\to\infty$, it is natural to ask what happens to Theorem~\ref{mpc_suboptimality_theorem} in the continuous-time limit. The bound \eqref{mpc_suboptimality_bound} remains well-behaved as $N\to\infty$. It suffices to show that the one-step consistency errors satisfy $\varepsilon_i=\mathcal{O}(\Delta t_i)$. Under the additional assumptions that $t\mapsto \mathcal M_{t}^\theta(x)$ is $L_{\mathcal{M}_t,i}$-Lipschitz on the region of interest (uniformly in $x$), and that the velocity is bounded on \(\widehat{\mathcal{S}}_i\), this scaling holds. The derivation is given in Appendix~\ref{app:theory_proofs}.
\end{remark}

Next, for completeness, we state the equivalence of Problem~\ref{receding_horizon_problem_original}, Problem~\ref{receding_horizon_problem_simplified}, and Problem~\ref{receding_horizon_problem_reversed} in the following theorem.

\begin{theorem}
\label{equivalence_theorem}
For any initial state \(x_0=\bar{x}_0\), running Problem~\ref{receding_horizon_problem_original}, Problem~\ref{receding_horizon_problem_simplified}, or Problem~\ref{receding_horizon_problem_reversed} produces the same control trajectory \(\{u_i\}_{i=0}^{N-1}\) and state trajectory \(\{x_i\}_{i=0}^N\).
\end{theorem}

Lastly, we analyze the gap between Problem~\ref{receding_horizon_problem_reversed} and Problem~\ref{receding_horizon_problem_surrogate}, which is introduced by the one-step fixed-point iteration. Both problems share the same decision variable $\widehat{x}_N$ and the same constraints $h(\widehat{x}_N)\le 0$. The difference lies in the quadratic penalty term in the objective. Problem~\ref{receding_horizon_problem_reversed} uses \(\|(\mathcal{M}_{t_{i+1}}^{\theta})^{-1}(\widehat{x}_{N})-\bar{x}_{i+1}\|_2^2\), whereas Problem~\ref{receding_horizon_problem_surrogate} uses \(\|\mathcal{F}_{t_{i+1}}^{\theta}(\widehat{x}_{N})-\bar{x}_{i+1}\|_2^2\). Assume \(\mathcal{W}_{t_i}^\theta(\cdot)\) is \(L_{\mathcal{W}_x,i}\)-Lipschitz continuous on the region of interest. Assume \(\left|\beta_{t_{i+1}}\right|L_{\mathcal{W}_x,i+1}<1\) for \(i=0,1,\cdots,N-1\). We have the following result.

\begin{theorem}
\label{fixed_point_error_theorem}
For \(i=0,1,\cdots,N-1\) and any \(y\in \mathbb{R}^d\), denote the subproblem objective in Problem~\ref{receding_horizon_problem_reversed} as \(\mathcal{J}_i^{\textup{P5}}(y)\coloneqq C(y)+\frac{\lambda_{\textup{oc}}}{2\Delta t_i}\|(\mathcal{M}_{t_{i+1}}^{\theta})^{-1}(y)-\bar{x}_{i+1}\|_2^2\), and that in Problem~\ref{receding_horizon_problem_surrogate} as \(\mathcal{J}_i^{\textup{P6}}(y)\coloneqq C(y)+\frac{\lambda_{\textup{oc}}}{2\Delta t_i}\|\mathcal{F}_{t_{i+1}}^{\theta}(y)-\bar{x}_{i+1}\|_2^2\). Then, we have
\begin{equation}
\label{fixed_point_objective_bound}
\left|\mathcal{J}_i^{\textup{P5}}(y)-\mathcal{J}_i^{\textup{P6}}(y)\right| \le \frac{\lambda_{\textup{oc}}}{2 \Delta t_i} \frac{r(2-r)}{(1-r)^2} \alpha_{t_{i+1}}^2\|y-\bar{y}\|^2,
\end{equation}
where \(r=\left|\beta_{t_{i+1}}\right| L_{\mathcal{W}_x,i+1}\) and \(\bar{y}=\mathcal{M}_{t_{i+1}}^{\theta}(\bar{x}_{i+1})\).
\end{theorem}

\begin{remark}
\label{practical_heuristic}
Theorem~\ref{fixed_point_error_theorem} requires the quantity \(\left|\beta_{t_{i+1}}\right| L_{\mathcal{W}_x,i+1}\) to be small. For affine conditional paths, the boundary condition \(\beta_1=0\) makes this requirement easier to satisfy as the sampling process approaches the final time \(t=1\). This is also consistent with Theorem~\ref{mpc_suboptimality_theorem}, where the suboptimality gap depends on one-step consistency errors governed by the accuracy of terminal state estimators \(\mathcal{M}_{t_i}^\theta(\cdot)\), which typically improve over time. These observations suggest a practical heuristic: skip control in the early stages of sampling, and activate the subproblems in later steps, when both the estimator and the fixed-point approximation are more reliable.
\end{remark}

\section{Experiments}

In this section, we evaluate our algorithm \algname\ on a diverse set of tasks from various domains, namely, robotic manipulation, maze navigation, partial differential equation (PDE) control, and text-guided image editing. Results across these distinct domains demonstrate the generality and effectiveness of our proposed framework.

In every task, a pretrained flow-matching model \(v_t^\theta(x)\) is provided. There are hard constraints \(h(x)\le 0\) (see Remark~\ref{constraints_form}), which may be complex and involve multiple equalities and inequalities. There are also costs \(C(x)\) capturing additional preferences over samples. Our goal is to generate samples that satisfy the hard constraints \(h(x)\le 0\) while minimizing the costs \(C(x)\). We will describe \(v_t^\theta(x)\), \(h(x)\), and \(C(x)\) in detail for each task later.

For comparison, we implement state-of-the-art baselines from the two families discussed in Related Work: soft-constrained and hard-constrained approaches. For the soft-constrained family, we add penalties for constraint violations to the original cost \(C(x)\) and then apply training-free guidance to minimize the augmented cost during sampling. Specifically, we implement \textbf{OC-Flow} \cite{wang2025training}, which perturbs the sampling trajectory via indirect optimal control, and a standard gradient-based guidance approach (hereafter \textbf{Gradient Guidance}) used widely in prior work \cite{guo2024gradient, feng2025on, pandey2025variational}, which updates samples using the cost evaluated at the posterior mean. For the hard-constrained family, the central strategy is projection during sampling, with most prior work focused on diffusion models. We consider three representative projection-based methods and adapt them to flow-matching models: (i) \textbf{Projection-All} \cite{christopher2024constrained, romer2025diffusion, zampini2025training, luan2025projected}, projecting after every sampling step; (ii) \textbf{Projection-Late} \cite{yuan2023physdiff, bouvier2025ddat}, projecting only in the later sampling steps; and (iii) \textbf{Projection-Relaxed} \cite{zhang2025constrained, xiao2025safediffuser}, which performs a fixed number of augmented Lagrangian iterations at each sampling step, thereby relaxing constraints in early sampling steps when the multipliers are small.
Since projection-based methods only handle constraint enforcement, we also report enhanced versions that combine them with Gradient Guidance on the cost \(C(x)\) during sampling, enabling a fair comparison with our approach, which addresses cost minimization and constraint satisfaction in a unified framework. Lastly, we include standard sampling without any guidance, denoted as \textbf{Original}, to demonstrate the capabilities of the base model.

\revised{All experiments were conducted on a desktop equipped with an Intel Core Ultra 9 285K CPU and an NVIDIA GeForce RTX 5090 GPU. All computations involving neural networks, including model inference and Jacobian calculations, were performed on the GPU. Across all tasks, we use the scheduler \((\alpha_t=t,\beta_t=1-t)\) and Euler discretization with uniform timesteps, which is a simple and widely adopted choice in the flow-matching literature. Comprehensive experimental details are provided in Appendix~\ref{app:experiment_details}. To better illustrate the steering behavior of \algname, we provide additional visualizations in Appendix~\ref{app:steering_visualizations}. Additional baselines, including post-processing methods and the concurrent work SafeFlowMatcher \cite{yang2025safeflowmatcher}, are reported in Appendix~\ref{app:additional_baselines}. We also include results on sensitivity analysis and stress testing in Appendix~\ref{app:sensitivity_analysis}, as well as demonstrations of extensions to higher-order solvers and non-uniform time grids in Appendix~\ref{app:solver_schedule_results}.}

\subsection{Robotic Manipulation}

This task follows the setup of \cite{romer2025diffusion} and uses the D3IL benchmark introduced in \cite{jia2024towards}. A robotic manipulator aims to reach a target region with its end-effector while avoiding obstacles, as shown in Fig.~\ref{d3il_env}. The state \(s\in \mathbb{R}^{4}\) comprises the current and desired two-dimensional positions of the end-effector. The action \(a\in \mathbb{R}^{2}\) is the desired two-dimensional velocity of the end-effector. D3IL provides expert demonstrations that weave around six pillars to reach a target region. We train a flow-matching model to generate trajectories of horizon \(H\), i.e., samples \(x=(s_0, a_0, s_1, a_1, \cdots, s_{H-1}, a_{H-1})\in \mathbb{R}^{6H}\). At test time, conditioned on the current state \(s\), we sample a trajectory with \(s_0=s\) and execute the actions \(\left\{a_i\right\}_{i=0}^{T-1}\) sequentially in the environment, where \(T\le H\) denotes the replanning horizon. We also introduce novel test-time obstacles that conflict with many trajectories in the training data, thus increasing task difficulty. The hard constraints \(h(x)\le 0\) ensure that the trajectory avoids all obstacles, i.e., both the pillars and the purple regions. Since a sample \(x\) is a trajectory of the robotic system, we also incorporate dynamics constraints into \(h(x)\le 0\) to ensure physical fidelity \cite{bouvier2025ddat, romer2025diffusion}: \(s_{i+1}=f(s_i,a_i)\) for \(i=0,1,\cdots,H-1\), where \(f\) is fitted from the training data. The cost \(C(x)\) is the squared distance between the final state \(s_{H-1}\) and the target region, which encourages reaching the target as fast as possible. The optimization problem \eqref{final_optimization_formulation} in \algname, as well as the projection operations in the baselines, is solved using the open-source nonlinear programming solver IPOPT \cite{wachter2006implementation}.

\begin{figure}[htbp]
    \centering
    \includegraphics[width=1.5in]{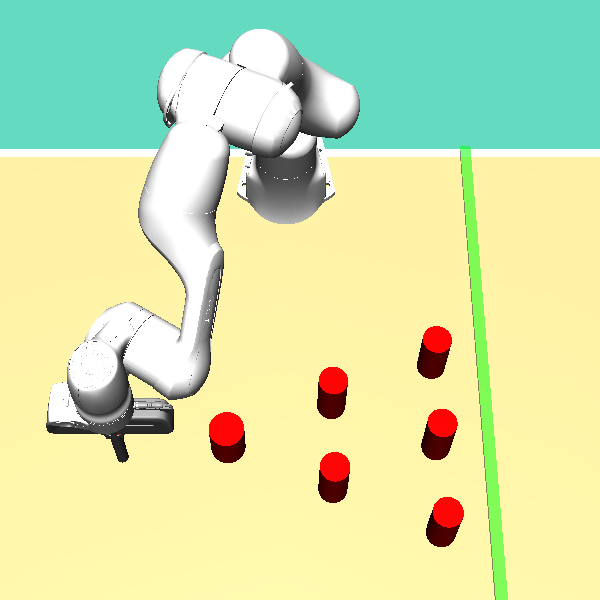}
    \caption{The robotic manipulation task.}
    \label{d3il_env}
\end{figure}

We evaluate all methods over 50 trials. The simulation starts at the same initial position and terminates when the end-effector reaches the target area or collides with any obstacle. A visualization of our algorithm is shown in Fig.~\ref{d3il_trial}. We record three metrics: (i) Safety Rate: the percentage of trials without collision; (ii) Total Steps (Safe Trials): the average number of steps that the robot takes to reach the target among trials not terminated by collision; and (iii) Computation Time: the average time taken to sample a trajectory from the flow-matching model at each replanning step. The quantitative results are summarized in Table~\ref{d3il_results}. Our algorithm \algname\ is the only method to achieve a perfect safety rate (1.00), while also requiring the fewest steps to reach the target and incurring only mild computational overhead. In contrast, other baselines exhibit much lower safety rates and longer paths.

\begin{figure}[htbp]
    \centering
    \includegraphics[width=3.2in]{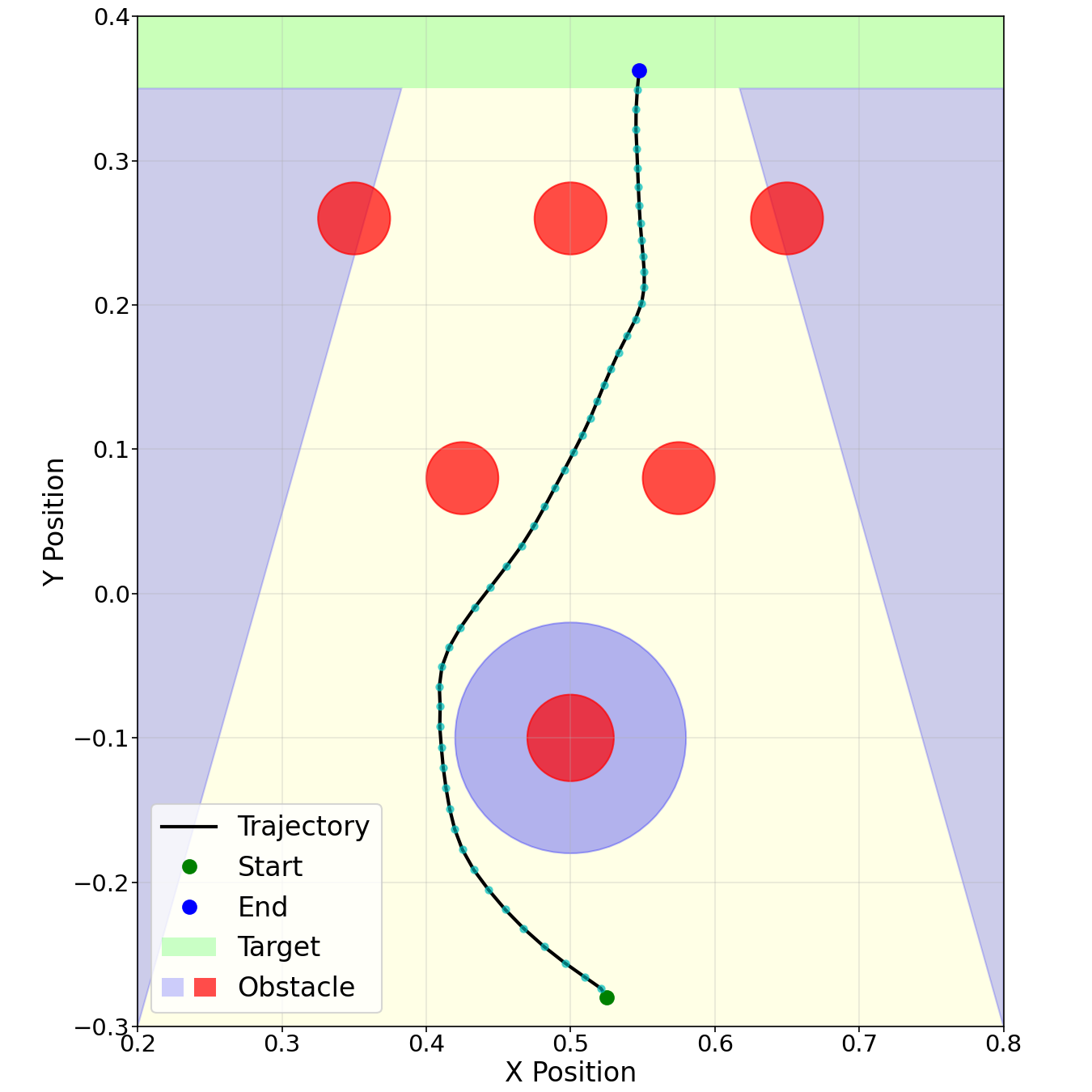}
    \caption{Visualization of \algname\ in the robotic manipulation task. The green region denotes the target area, the red circles represent original obstacles, and the purple region denotes the novel obstacles introduced at test time. The robot successfully weaves around all obstacles to reach the target.}
    \label{d3il_trial}
\end{figure}

\begin{table*}[htbp]
    \centering
    \caption{Results on the robotic manipulation task. Values of the form \(a\pm b\) indicate mean \(\pm\) standard deviation. \algname\ is the only method to achieve a perfect safety rate (1.00), while also requiring the fewest steps to reach the target and incurring only mild computational overhead. In contrast, other baselines exhibit much lower safety rates and longer paths.}
    \label{d3il_results}
    \setlength{\tabcolsep}{6pt}
    \begin{tabular}{l c c c}
        \toprule
        \textbf{Method} & \textbf{Safety Rate} & \textbf{Total Steps (Safe Trials)} & \textbf{Computation Time (s)} \\
        \midrule
        Original & 0.06 & 58.7 $\pm$ 4.0 & 0.060 $\pm$ 0.009 \\
        Gradient Guidance & 0.18 & 61.7 $\pm$ 6.3 & 0.992 $\pm$ 0.022 \\
        OC-Flow & 0.14 & 63.9 $\pm$ 6.3 & 0.847 $\pm$ 0.016 \\
        Projection-All & 0.46 & 67.2 $\pm$ 7.1 & 0.349 $\pm$ 0.051 \\
        Projection-Late & 0.76 & 67 $\pm$ 11 & 0.236 $\pm$ 0.072 \\
        Projection-Relaxed & 0.10 & 63.4 $\pm$ 7.9 & 0.116 $\pm$ 0.014 \\
        Projection-All + Gradient Guidance & 0.40 & 70.6 $\pm$ 8.3 & 1.556 $\pm$ 0.057 \\
        Projection-Late + Gradient Guidance & 0.68 & 67.2 $\pm$ 9.3 & 1.380 $\pm$ 0.035 \\
        Projection-Relaxed + Gradient Guidance & 0.04 & 65.0 $\pm$ 2.8 & 1.074 $\pm$ 0.016 \\
        \algname\ (ours) & \textbf{1.00} & \textbf{52.5} $\pm$ \textbf{4.4} & 0.190 $\pm$ 0.023 \\
        \bottomrule
    \end{tabular}
\end{table*}


\subsection{Maze Navigation}

This task follows the setup of \cite{janner2022planning, xiao2025safediffuser, zhang2025constrained}, which are based on the Maze2D environment from D4RL \cite{fu2020d4rl}. As shown in Fig.~\ref{maze_trial}, a force-actuated ball navigates a two-dimensional maze to a goal position. The state \(s\in \mathbb{R}^{4}\) consists of the two-dimensional position and velocity of the ball. The action \(a\in \mathbb{R}^{2}\) is the applied two-dimensional force. We train a flow-matching model to generate trajectories of horizon \(H\), i.e., samples \(x=(s_0, a_0, s_1, a_1, \cdots, s_{H-1}, a_{H-1})\in \mathbb{R}^{6H}\). At test time, a trajectory is sampled conditioned on the initial state \(s_0\) and the goal state \(s_{H-1}\). A proportional-derivative (PD) controller is deployed to track the positions in this trajectory. The hard constraints \(h(x)\le0\) enforce obstacle avoidance for the two red regions in Fig.~\ref{maze_trial}. As in the robotic manipulation task, we impose dynamics constraints for physical fidelity: \(s_{i+1}=f(s_i,a_i)\) for \(i=0,1,\cdots,H-1\), where \(f\) is fitted from the training data. We use IPOPT to solve the optimization problem \eqref{final_optimization_formulation} and the projection steps in the baselines. The cost \(C(x)\) is the (squared) total length of the trajectory, encouraging fast progress to the goal.

We evaluate all methods over 50 trials. A visualization of our algorithm is shown in Fig.~\ref{maze_trial}. We report four metrics: (i) Safety Rate: the percentage of trials without entering the red obstacles; (ii) Violations: the average number of timesteps during which the ball is inside the obstacles; (iii) Score: the D4RL normalized score \cite{fu2020d4rl} (faster goal reaching yields higher scores); and (iv) Computation Time: the average time taken to sample a trajectory from the flow-matching model. The quantitative results are summarized in Table~\ref{maze2d_results}. Our method \algname\ is the only approach to achieve a perfect safety rate (1.00) with zero violations while also obtaining the highest score. Its computation time is comparable to or lower than most baselines.

\begin{figure}[htbp]
    \centering
    \includegraphics[width=2.5in]{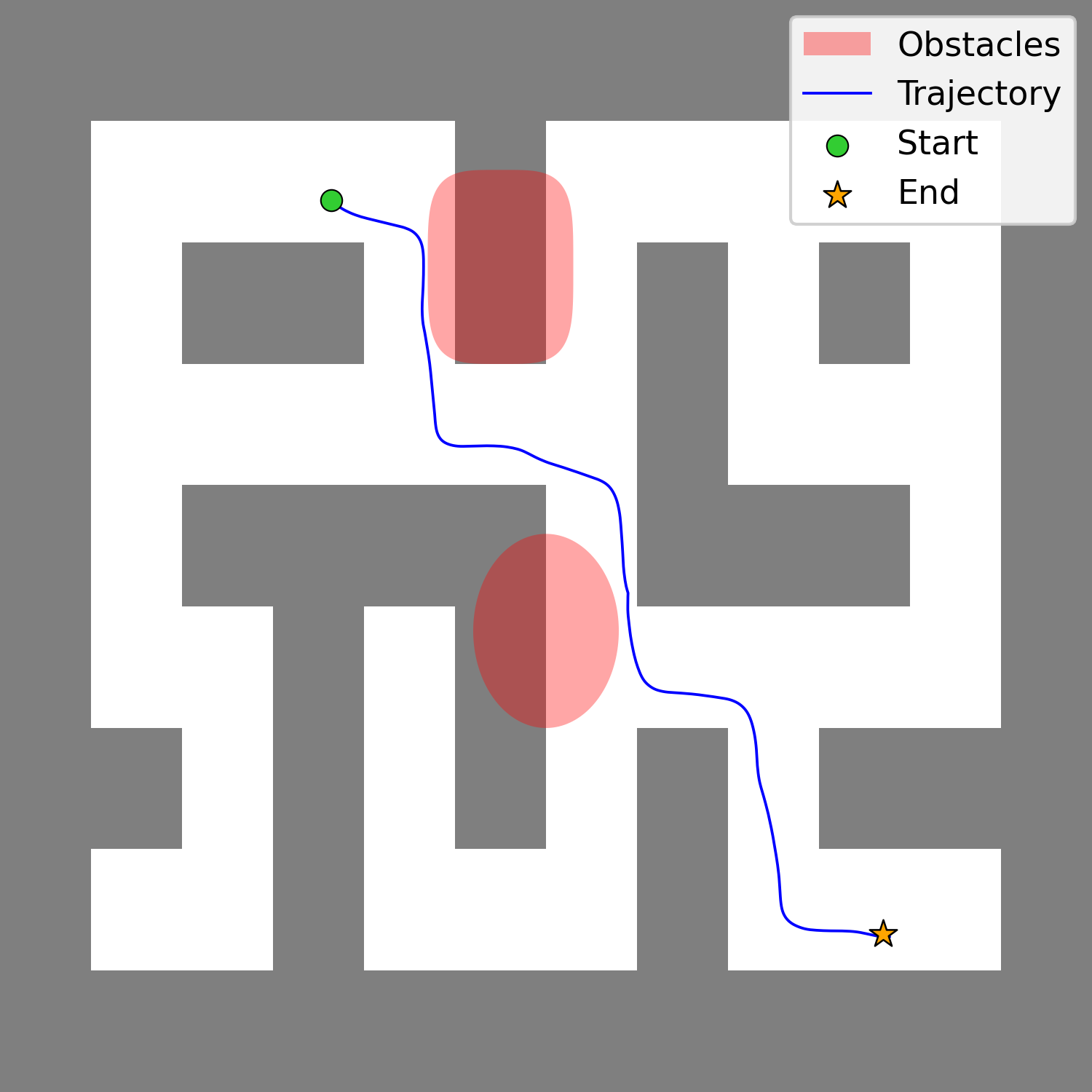}
    \caption{Visualization of \algname\ in the maze navigation task. The red region denotes the introduced obstacles. The robot successfully navigates from the start position to the goal position while avoiding obstacles.}
    \label{maze_trial}
\end{figure}

\begin{table*}[htbp]
    \centering
    \caption{Results on the maze navigation task. Values of the form \(a\pm b\) indicate mean \(\pm\) standard deviation. \algname\ is the only approach to achieve a perfect safety rate (1.00) with zero violations while also obtaining the highest score. Its computation time is comparable to or lower than most baselines.}
    \label{maze2d_results}
    \setlength{\tabcolsep}{6pt}
    \begin{tabular}{l c c c c}
        \toprule
        \textbf{Method} & \textbf{Safety Rate} & \textbf{Violations} & \textbf{Score} & \textbf{Computation Time (s)} \\
        \midrule
        Original & 0.02 & 49 $\pm$ 20 & 1.47 $\pm$ 0.45 & 0.082 $\pm$ 0.034 \\
        Gradient Guidance & 0.88 & 1.4 $\pm$ 4.2 & 1.11 $\pm$ 0.75 & 6.666 $\pm$ 0.076 \\
        OC-Flow & 0.04 & 35 $\pm$ 17 & 1.47 $\pm$ 0.45 & 3.983 $\pm$ 0.085 \\
        Projection-All & 0.22 & 38 $\pm$ 30 & 1.41 $\pm$ 0.54 & 10.4 $\pm$ 1.6 \\
        Projection-Late & 0.30 & 34 $\pm$ 31 & 1.51 $\pm$ 0.39 & 6.2 $\pm$ 1.2 \\
        Projection-Relaxed & 0.00 & 33 $\pm$ 14 & 1.53 $\pm$ 0.32 & 4.11 $\pm$ 0.33 \\
        Projection-All + Gradient Guidance & 0.24 & 37 $\pm$ 30 & 1.56 $\pm$ 0.34 & 25.7 $\pm$ 3.1 \\
        Projection-Late + Gradient Guidance & 0.30 & 34 $\pm$ 30 & 1.50 $\pm$ 0.39 & 18.02 $\pm$ 0.74 \\
        Projection-Relaxed + Gradient Guidance & 0.00 & 25 $\pm$ 10 & 1.53 $\pm$ 0.32 & 11.20 $\pm$ 0.19 \\
        \algname\ (ours) & \textbf{1.00} & \textbf{0.0} $\pm$ \textbf{0.0} & \textbf{1.620} $\pm$ \textbf{0.010} & 4.09 $\pm$ 0.82 \\
        \bottomrule
    \end{tabular}
\end{table*}

\subsection{PDE Control}

One-dimensional Burgers' equation is a fundamental PDE that models various physical phenomena, such as fluid dynamics and traffic flow. Following \cite{hu2025from}, we consider Dirichlet boundary conditions, with state \(u(t,s)\) and external control \(f(t,s)\):
\begin{equation}
\label{burgers}
\begin{cases}\frac{\partial u}{\partial t} + u \frac{\partial u}{\partial s} = \nu \frac{\partial^2 u}{\partial s^2} + f(t,s) & (t,s)\in[0, T] \times [0,L] \\ u(0,s) = u_0(s), u(T,s) = u_T(s) & s\in[0,L] \\ u(t,0) = u(t,L) = 0 & t\in[0,T]\end{cases}
\end{equation}
We train a flow-matching model to generate solutions to PDE \eqref{burgers}. Each sample \(x\) contains the discretized states and controls on a spatiotemporal grid with \(m\) time steps and \(n\) spatial points. The dataset is from \cite{hu2025from}, containing PDE solutions under diverse initial and terminal conditions. The hard constraints \(h(x)\le0\) enforce time-varying bounds on $|u|$. For physical fidelity, we also enforce a discretized form of \eqref{burgers} on each sample \(x\), where the viscosity parameter \(\nu\) is treated as uncertain within a known interval, i.e, \(\nu\in[\nu_{\textup{min}},\nu_{\textup{max}}]\). The cost \(C(x)\) is the total control energy. We solve \eqref{final_optimization_formulation} in \algname\ and the projection steps in the baselines with IPOPT.

Conditioned on given boundary conditions \(u(0,s)=u_0(s)\) and \(u(T,s)=u_T(s)\), we sample from the flow-matching model, extract the predicted control \(f\), and simulate the PDE on a fine grid to obtain the resulting controlled state \(u\). We evaluate all methods over 50 trials, each with different boundary conditions. A visualization of controlled state \(u\) and predicted control \(f\) at one trial is shown in Fig.~\ref{pde_trial}. State slices at three time instants are presented in Fig.~\ref{pde_slices}, demonstrating that the controlled state complies with the time-varying bounds. Following \cite{hu2025from}, we report three metrics quantifying constraint satisfaction of the state \(u\): (i) \(\mathcal{R}_{\text{sample}}\): the fraction of trials with any violation; (ii) \(\mathcal{R}_{\text{time}}\): the fraction of unsafe timesteps over all timesteps; and (iii) \(\mathcal{R}_{\text{point}}\): the fraction of spatial points that ever violate a constraint across all timesteps. We also report Control Energy and Computation Time. The quantitative results are summarized in Table~\ref{pde_control_results}. \algname\ is one of five methods that achieve perfect constraint satisfaction (zero on all three safety metrics). Among them, \algname\ attains the lowest control energy and the shortest computation time.

\begin{figure}[htbp]
    \centering
    \includegraphics[width=3.2in]{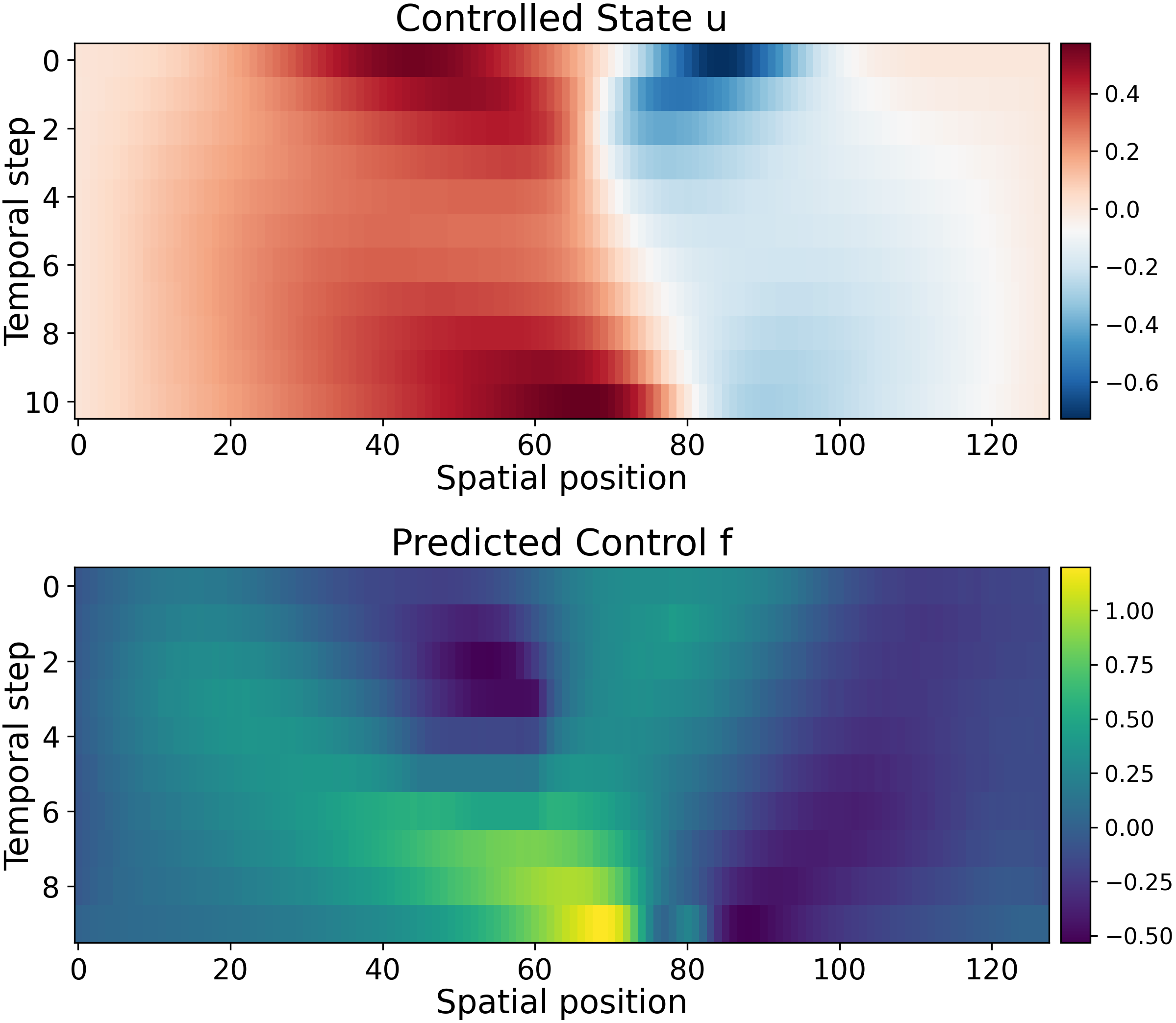}
    \caption{Heatmaps of controlled state \(u\) and predicted control \(f\) produced by \algname\ at one trial in the PDE control task. The grid uses \(m=10\) time steps and \(n=128\) spatial points; the axis ticks reflect this discretization.}
    \label{pde_trial}
\end{figure}

\begin{figure}[htbp]
    \centering
    \includegraphics[width=3.4in]{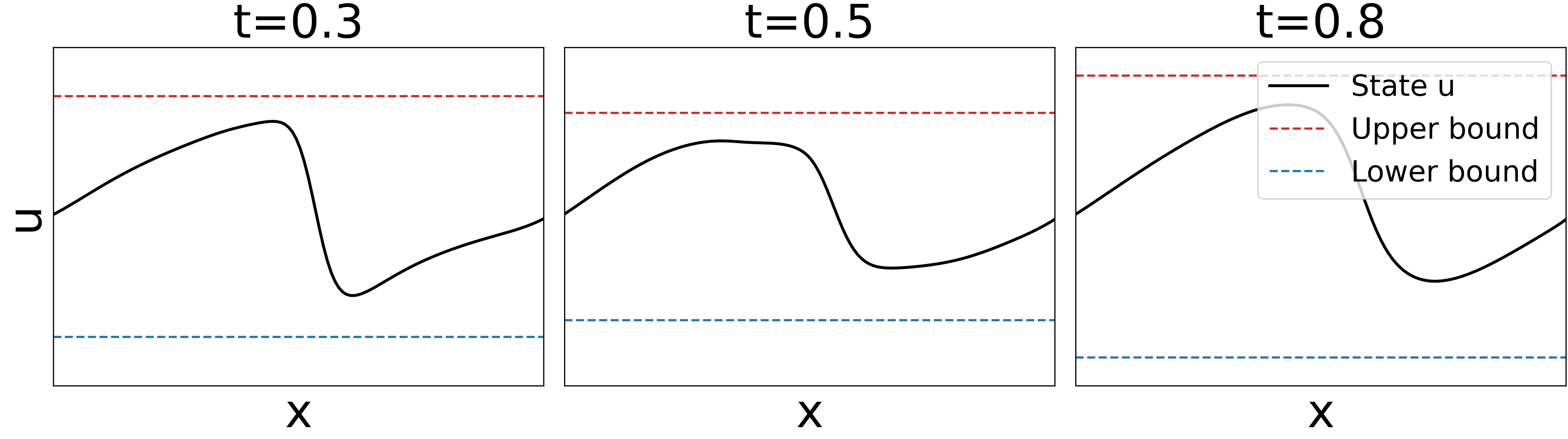}
    \caption{Slices of controlled state \(u\) at three time instants, produced by \algname. The state constraints are satisfied.}
    \label{pde_slices}
\end{figure}

\begin{table*}[htbp]
    \centering
    \caption{Results on the PDE control task. \(\mathcal{R}_{\text{sample}}\) denotes the fraction of trials with any violation; \(\mathcal{R}_{\text{time}}\) denotes the fraction of unsafe timesteps over all timesteps; and \(\mathcal{R}_{\text{point}}\) denotes the fraction of spatial points that ever violate a constraint across all timesteps. Values of the form \(a\pm b\) indicate mean \(\pm\) standard deviation. \algname\ is one of five methods that achieve perfect constraint satisfaction (zero on all three safety metrics). Among them, \algname\ attains the lowest control energy and the shortest computation time.}
    \label{pde_control_results}
    \setlength{\tabcolsep}{6pt}
    \begin{tabular}{l c c c c c}
        \toprule
        \textbf{Method} & $\boldsymbol{\mathcal{R}_{\textbf{sample}}}$ & $\boldsymbol{\mathcal{R}_{\textbf{time}}}$ & $\boldsymbol{\mathcal{R}_{\textbf{point}}}$ & \textbf{Control Energy} & \textbf{Computation Time (s)} \\
        \midrule
        Original & 1.00 & 0.77 & 0.48 & 0.59 $\pm$ 0.26 & 0.162 $\pm$ 0.085 \\
        Gradient Guidance & 0.36 & 0.13 & 0.04 & 0.31 $\pm$ 0.10 & 2.71 $\pm$ 0.15 \\
        OC-Flow & 0.78 & 0.47 & 0.19 & 0.180 $\pm$ 0.075 & 4.434 $\pm$ 0.034 \\
        Projection-All & \textbf{0.00} & \textbf{0.00} & \textbf{0.00} & 0.53 $\pm$ 0.16 & 13.6 $\pm$ 1.2 \\
        Projection-Late & \textbf{0.00} & \textbf{0.00} & \textbf{0.00} & 0.51 $\pm$ 0.16 & 8.9 $\pm$ 2.7 \\
        Projection-Relaxed & 1.00 & 0.72 & 0.42 & 0.49 $\pm$ 0.21 & 0.20 $\pm$ 0.17 \\
        Projection-All + Gradient Guidance & \textbf{0.00} & \textbf{0.00} & \textbf{0.00} & 0.36 $\pm$ 0.12 & 17.2 $\pm$ 1.3 \\
        Projection-Late + Gradient Guidance & \textbf{0.00} & \textbf{0.00} & \textbf{0.00} & 0.36 $\pm$ 0.12 & 11.5 $\pm$ 2.5 \\
        Projection-Relaxed + Gradient Guidance & 0.94 & 0.64 & 0.30 & 0.27 $\pm$ 0.12 & 2.84 $\pm$ 0.20 \\
        \algname\ (ours) & \textbf{0.00} & \textbf{0.00} & \textbf{0.00} & \textbf{0.28} $\pm$ \textbf{0.10} & 8.3 $\pm$ 1.1 \\
        \bottomrule
    \end{tabular}
\end{table*}

\subsection{Text-Guided Image Editing}

In this task, leveraging a flow-matching model pretrained on the CelebA-HQ celebrity face dataset \cite{karras2018progressive}, we edit an input image according to a text prompt so that the edited image aligns with the prompt requirements. We follow the setup of \cite{wang2025training}. The pretrained flow-matching model is from \cite{liu2023flow}. Given an input image, we integrate the flow ODE backward from \(t=1\) to \(t=0\) to recover the corresponding initial noise \(x_0\), then apply a guided sampling method to generate the edited image \(x_1\). We use CLIP \cite{radford2021learning} to score how well an image matches the text prompt and define the cost \(C(x)\) as the negative CLIP score. However, a potential risk is that the edits are so substantial that the original identity of the input image is lost. To address this issue, we use LPIPS \cite{zhang2018unreasonable} to quantify perceptual similarity between the edited and original images, and impose a hard constraint \(h(x)\le0\) requiring LPIPS to remain below an upper bound. We test with five text prompts: ``A photo of an angry face.'', ``A photo of a face with curly hair.'', ``A photo of an old face.'', ``A photo of a sad face.'', and ``A photo of a smiling face.'', which are referred to as Angry, Curly, Old, Sad, and Smile hereafter.

We evaluate on 200 images randomly sampled from the CelebA-HQ validation set. In this image editing setting, the Original baseline is not meaningful, as it simply returns the input image unchanged. The constraints-only baselines (Projection-All/Late/Relaxed) are also inapplicable: without optimizing the cost (i.e., CLIP score), there is no incentive to modify the image, so the LPIPS constraint is trivially satisfied. In addition, since the LPIPS constraint is computed by a neural network and the sampling uses many discretization steps (\(N=100\)), exact projection at every step or many steps, as in Projection-All/Late, is computationally prohibitive. Therefore, we compare \algname\ with Gradient Guidance, OC-Flow, and Projection-Relaxed + Gradient Guidance. The optimization problem \eqref{final_optimization_formulation} is solved with a fixed number of augmented Lagrangian iterations. Quantitative results by prompt are shown in Fig.~\ref{image_results}, and aggregated results are summarized in Table~\ref{image_metrics}. Safety Rate represents the fraction of edited images satisfying the LPIPS constraint. \algname\ is the only method to achieve 100\% constraint satisfaction across all prompts and images. OC-Flow and Gradient Guidance substantially violate the LPIPS constraint, indicating over-editing that alters identity. Projection-Relaxed + Gradient Guidance drives LPIPS unnecessarily low, even though the requirement is only to stay below the upper bound, which in turn hurts its CLIP score. \algname\ achieves the second-best CLIP score, very close to OC-Flow, and is much more computationally efficient than all baselines.

\begin{figure*}[htbp]
    \centering
    \includegraphics[width=6.8in]{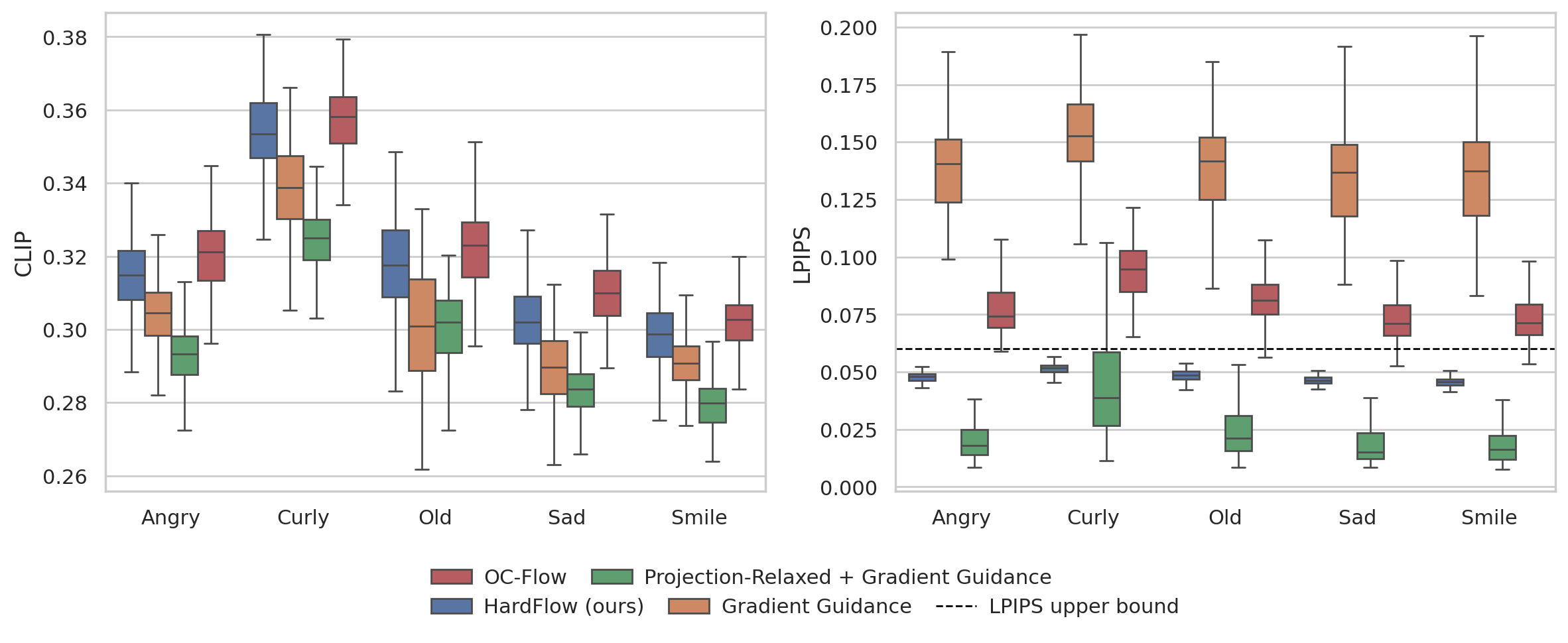}
    \caption{Box-and-whisker plots of CLIP (left; higher is better) and LPIPS (right; below the dashed line is required) for five prompts on the text-guided image editing task. For each prompt and method, the box spans the interquartile range (25th to 75th percentiles), the horizontal line marks the median, and the whiskers extend to 1.5 times the interquartile range. The dashed line in the LPIPS panel indicates the LPIPS upper bound.}
    \label{image_results}
\end{figure*}

\begin{table*}[htbp]
  \centering
  \caption{Results on the text-guided image editing task. CLIP measures editing effectiveness (higher is better). LPIPS measures perceptual similarity between the edited and original images, and is required to be below 0.06. Safety rate denotes the percentage of edits that satisfy the constraint. \algname\ is the only method to achieve 100\% constraint satisfaction across all prompts and images, attains strong CLIP scores, and is substantially faster than all baselines.}
  \label{image_metrics}
  \setlength{\tabcolsep}{6pt}
  \begin{tabular}{l c c c c}
    \toprule
    \textbf{Method} & \textbf{Safety Rate} & \textbf{LPIPS} & \textbf{CLIP} & \textbf{Computation Time (s)} \\
    \midrule
    OC-Flow & 0.033 & 0.082 $\pm$ 0.009 & 0.322 $\pm$ 0.021 & 158.302 $\pm$ 0.016 \\
    Gradient Guidance & 0.000 & 0.144 $\pm$ 0.007 & 0.304 $\pm$ 0.019 & 131.819 $\pm$ 0.076 \\
    Projection-Relaxed + Gradient Guidance & 0.939 & 0.026 $\pm$ 0.011 & 0.296 $\pm$ 0.018 & 129.336 $\pm$ 0.047 \\
    HardFlow (ours) & \textbf{1.000} & \textbf{0.048} $\pm$ \textbf{0.002} & \textbf{0.317} $\pm$ \textbf{0.022} & 51.294 $\pm$ 0.004 \\
    \bottomrule
  \end{tabular}
\end{table*}

To further illustrate the benefits of imposing hard constraints on LPIPS, we provide qualitative comparisons in Fig.~\ref{image_comparison}. The three reference images are shown in the top-left of Fig.~\ref{image_comparison}, with their CelebA-HQ numbers displayed beneath each. We refer to them by these numbers hereafter. A notable issue with OC-Flow is that it appears to change the perceived gender of the subject, for example, in the Angry, Curly, and Sad edits of 002140, and the Curly and Smile edits of 014314. Both OC-Flow and Gradient Guidance can alter facial features excessively, deviating too much from the original identity, for example, the Angry and Smile edits of 004666 by OC-Flow and most edits of 004666 and 014314 by Gradient Guidance. Gradient Guidance also changes hair color drastically in all edits of 002140, which is unnecessary for the given prompts. Projection-Relaxed + Gradient Guidance yields much worse visual quality despite decent CLIP and LPIPS values, suggesting reward hacking in which the numerical scores appear acceptable even though the outputs are visually degraded. This may stem from constraining intermediate states rather than the final state in the sampling process.

\begin{figure*}[htbp]
    \centering
    \includegraphics[width=7.0in]{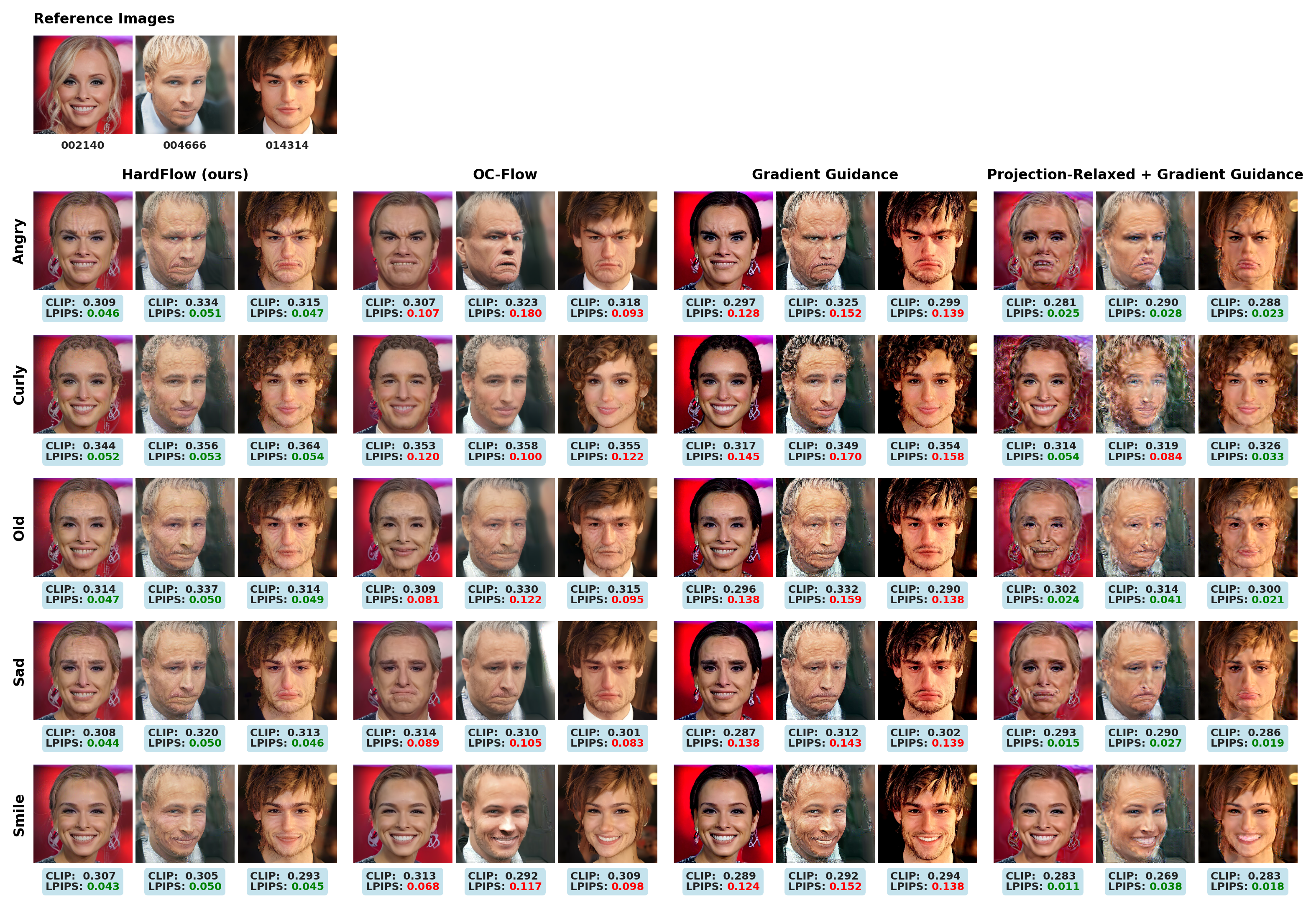}
    \caption{Qualitative comparison of different methods on the text-guided image editing task. The three reference images are shown in the top-left, with their CelebA-HQ numbers displayed beneath each. Below each edited image, CLIP and LPIPS are shown. LPIPS values that satisfy the constraint are displayed in green, and violations are displayed in red.}
    \label{image_comparison}
\end{figure*}

\begin{revisedblock}

\section{Limitations and Future Work}

In this section, we highlight several natural directions for future development.

First, our work focuses on hard-constrained sampling for pretrained flow-matching models in a training-free setting, i.e., without updating the model parameters. This setting is well motivated and well established in the literature, and training-free guidance is attractive because of its simplicity, generality, and ease of integration with existing pretrained models. A natural next step is to extend the framework to settings in which the model itself can also be updated, for example, through fine-tuning, so that constraint satisfaction and sample quality can be improved in a more adaptive manner.

Second, as a training-free approach, our method introduces additional computation at inference time in exchange for enforcing constraints without modifying the underlying model parameters. In many applications, this trade-off is worthwhile, but it also points to an important future direction: improving efficiency by developing variants that absorb part of the runtime optimization effort into the learning process.

Third, although our experiments already cover a diverse set of application domains and provide broad empirical support for the method, an important direction for future work is to evaluate the framework in even more complex and challenging settings, such as large-scale text-to-image generation, vision-based and contact-rich manipulation, and other demanding real-world applications.
An especially interesting direction for future work is to extend the framework to settings with high-dimensional visual inputs and contact-rich dynamics, where the benefits of hard-constrained sampling may be even more pronounced.

\end{revisedblock}

\section{Conclusion}

In this paper, we addressed the critical challenge of hard-constrained sampling for flow-matching models. We proposed a method that departs from common projection-based techniques by reformulating the sampling process as a trajectory optimization problem. The central idea is to utilize numerical optimal control to guide the generation trajectory so that constraints are met precisely at the final step, thereby avoiding the unnecessary restrictions and quality degradation associated with constraining the entire path. By leveraging the intrinsic structure of flow-matching models in conjunction with techniques from model predictive control, we obtain a tractable and efficient algorithm for an otherwise complex control problem. This control-theoretic perspective offers significant flexibility beyond mere feasibility, enabling the integration of auxiliary objectives that minimize distribution shift and enhance sample quality within a unified framework. The soundness of our approach is supported by theoretical analysis, and its practical efficacy and robustness are confirmed through extensive experiments across diverse domains. Empirically, our technique consistently and substantially outperforms existing methods, delivering superior results in both constraint satisfaction and sample quality. This work opens promising new avenues for applying trajectory optimization techniques, particularly direct methods, to generative models, presenting a powerful and flexible paradigm for controllable generation.

\section*{Acknowledgment}

We sincerely thank Wei-Cheng Huang for insightful discussions on optimal control and trajectory optimization.

\bibliographystyle{IEEEtran}
\bibliography{reference}

\vfill

\clearpage
\appendices

\section{Proofs for Theoretical Results}
\label{app:theory_proofs}

\subsection{Proof of Proposition~\ref{safety_guarantee}}

\begin{proof}
At the final iteration \(i = N-1\), let \(\widehat{x}_N^{*}\) denote a solution to \eqref{final_optimization_formulation}. Then \(h(\widehat{x}_N^{*}) \le 0\). Also note that \(t_N = 1\). The boundary conditions of the affine scheduler \((\alpha_t, \beta_t)\) at \(t = 1\) are \(\alpha_1 = 1\) and \(\beta_1 = 0\).  

Therefore, for \(i = N-1\), we can simplify the update rule in Line~\ref{compute_next_state} of Algorithm~\ref{main_algorithm} as
\begin{equation}
\nonumber
x_N 
= \alpha_1 \widehat{x}_N^{*} 
  + \beta_1 \frac{-\dot{\alpha}_1 \bar{x}_N + \alpha_1 v_1^\theta(\bar{x}_N)}
                 {\alpha_1 \dot{\beta}_1 - \dot{\alpha}_1 \beta_1}
= \widehat{x}_N^{*}.
\end{equation}
Consequently, \(h(x_N) = h(\widehat{x}_N^{*}) \le 0\), which shows that the output sample \(x_N\) satisfies the hard constraints.
\end{proof}

\subsection{Proof of Proposition~\ref{bellman_recursion}}

\begin{proof}
The result follows from Bellman's principle of optimality and standard dynamic programming arguments. See, e.g., \cite{bertsekas2012dynamic} for details.
\end{proof}

\subsection{Proof of Proposition~\ref{bellman_operator_nonexpansive}}

\begin{proof}
Let \(y_x\in \widehat{\mathcal{S}}_{i+1}\) attain the minimum in \((T_i V_{i+1})(x)=\frac{\lambda_{\textup{oc}}}{2 \Delta t_i}\left\|y_x-\Psi_i^{\theta}(x)\right\|_2^2+V_{i+1}(y_x)\). Then
\begin{equation}
\nonumber
T_i \widehat{V}_{i+1}(x) \leq \frac{\lambda_{\textup{oc}}}{2 \Delta t_i}\left\|y_x-\Psi_i^{\theta}(x)\right\|_2^2+\widehat{V}_{i+1}\left(y_x\right).
\end{equation}
Therefore,
\begin{equation}
\nonumber
\begin{aligned}
T_i \widehat{V}_{i+1}(x)-T_i V_{i+1}(x)
&\leq \widehat{V}_{i+1}\left(y_x\right)-V_{i+1}\left(y_x\right) \\
&\leq \sup _{y \in \widehat{\mathcal{S}}_{i+1}}\left(\widehat{V}_{i+1}-V_{i+1}\right)(y).
\end{aligned}
\end{equation}
\end{proof}

\subsection{Proof of Theorem~\ref{mpc_suboptimality_theorem}}

\begin{proof}
Denote $\left\{x_i^{\textup{P3}}\right\}_{i=0}^N$ as the state trajectory under $u^{\textup{P3}}$ starting from $x_0$. Since $u^{\textup{P3}}$ is feasible for Problem~\ref{receding_horizon_problem_original}, it then follows that $x_i^{\textup{P3}}\in \widehat{\mathcal{S}}_i$ for $i=1,2,\cdots,N$. By definition of \(u^{\textup{P3}}\), for \(i=0,1,\cdots,N-1\),
\begin{equation}
\nonumber
u_i^{\textup{P}3}=\underset{u}{\operatorname{argmin}}\left\{l_i(u)+\widehat{V}_{i+1}\left(\Psi_i^{\theta}\left(x_i^{\textup{P3}}, u\right)\right)\right\}.
\end{equation}
Evaluating the Bellman residual at these states, we have
\begin{equation}
\nonumber
\begin{aligned}
r_i\left(x_i^{\textup{P3}}\right)
&=T_{i}\widehat{V}_{i+1}\left(x_i^{\textup{P3}}\right)-\widehat{V}_i\left(x_i^{\textup{P3}}\right) \\
&=l_i\left(u_i^{\textup{P3}}\right)+\widehat{V}_{i+1}\left(x_{i+1}^{\textup{P3}}\right)-\widehat{V}_i\left(x_i^{\textup{P3}}\right).
\end{aligned}
\end{equation}
Summing over \(i=1,2,\cdots,N-1\) yields
\begin{equation}
\nonumber
\begin{aligned}
\sum_{i=1}^{N-1} r_i\left(x_i^{\textup{P3}}\right)
=& \sum_{i=1}^{N-1} l_i\left(u_i^{\textup{P3}}\right)+\sum_{i=1}^{N-1}\left(\widehat{V}_{i+1}\left(x_{i+1}^{\textup{P3}}\right)-\widehat{V}_i\left(x_i^{\textup{P3}}\right)\right) \\
=& \sum_{i=1}^{N-1} l_i\left(u_i^{\textup{P3}}\right)+\widehat{V}_N\left(x_N^{\textup{P3}}\right)-\widehat{V}_1\left(x_1^{\textup{P3}}\right) \\
=& \mathcal{J}\left(x_0, u^{\textup{P3}}\right)-l_0\left(u_0^{\textup{P3}}\right)-\widehat{V}_1\left(x_1^{\textup{P3}}\right).
\end{aligned}
\end{equation}

Since \((T_0 V_1)(x_0)=V_0(x_0)=\mathcal{J}(x_0,u^{\textup{P2}})\) and \(l_0\left(u_0^{\textup{P3}}\right)+\widehat{V}_1\left(x_1^{\textup{P3}}\right)=\left(T_0 \widehat{V}_1\right)\left(x_0\right)\), we have
\begin{equation}
\nonumber
\mathcal{J}(x_0,u^{\textup{P3}})-\mathcal{J}(x_0,u^{\textup{P2}})=\sum_{i=1}^{N-1} r_i\left(x_i^{\textup{P3}}\right)+ (T_0 \widehat{V}_1 - T_0 V_1)(x_0).
\end{equation}

By Proposition~\ref{bellman_operator_nonexpansive},
\begin{equation}
\nonumber
(T_0 \widehat{V}_1 - T_0 V_1)(x_0) \le \sup_{y\in \widehat{\mathcal{S}}_1} (\widehat{V}_1-V_1)(y).
\end{equation}

For any \(i=1,2,\cdots,N-1\) and \(x\in \widehat{\mathcal{S}}_i\), we have
\begin{equation}
\nonumber
\widehat{V}_i(x)-V_i(x)=T_i \widehat{V}_{i+1}(x)-T_i V_{i+1}(x)- r_i(x).
\end{equation}
Taking supremum over \(x\in \widehat{\mathcal{S}}_i\), we obtain
\begin{equation}
\nonumber
\begin{aligned}
\sup_{x\in \widehat{\mathcal{S}}_i}(\widehat{V}_i-V_i)(x)
& \leq \sup _{x\in \widehat{\mathcal{S}}_i}(T_i \widehat{V}_{i+1}-T_i V_{i+1})(x)+\left\|r_i\right\|_{\infty, \widehat{\mathcal{S}}_i} \\
& \leq \sup _{y \in \widehat{\mathcal{S}}_{i+1}}\left(\widehat{V}_{i+1}-V_{i+1}\right)(y)+\left\|r_i\right\|_{\infty, \widehat{\mathcal{S}}_i},
\end{aligned}
\end{equation}
where the second inequality follows from Proposition~\ref{bellman_operator_nonexpansive}. We use the notation \(\|f\|_{\infty, \mathcal{X}}=\sup_{x\in \mathcal{X}} |f(x)|\) for any function \(f\) and set \(\mathcal{X}\). Define
\begin{equation}
\nonumber
E_i\coloneqq\sup_{x\in \widehat{\mathcal{S}}_i}(\widehat{V}_i-V_i)(x), \qquad i=1,2,\cdots,N.
\end{equation}
Then we have the recursion
\begin{equation}
\nonumber
E_i \leq E_{i+1}+\left\|r_i\right\|_{\infty, \widehat{\mathcal{S}}_i}, \qquad i=1,2,\cdots,N-1,
\end{equation}
and \(E_N=0\) since \(\widehat{V}_N=V_N\). Unrolling the recursion yields
\begin{equation}
\nonumber
E_1 \leq \sum_{i=1}^{N-1}\left\|r_i\right\|_{\infty, \widehat{\mathcal{S}}_i}.
\end{equation}
Therefore,
\begin{equation}
\nonumber
\begin{aligned}
(T_0 \widehat{V}_1 - T_0 V_1)(x_0)
& \leq \sup_{y\in \widehat{\mathcal{S}}_1} (\widehat{V}_1-V_1)(y) \\
& = E_1 \leq \sum_{i=1}^{N-1}\left\|r_i\right\|_{\infty, \widehat{\mathcal{S}}_i}.
\end{aligned}
\end{equation}
Combining the above results, we arrive at
\begin{equation}
\nonumber
\begin{aligned}
\mathcal{J}(x_0,u^{\textup{P3}})-\mathcal{J}(x_0,u^{\textup{P2}})
& \leq \sum_{i=1}^{N-1}\left(\left\|r_i\right\|_{\infty, \widehat{\mathcal{S}}_i}+r_i\left(x_i^{\textup{P3}}\right)\right) \\
& \leq 2 \sum_{i=1}^{N-1}\left\|r_i\right\|_{\infty, \widehat{\mathcal{S}}_i}.
\end{aligned}
\end{equation}

It remains to bound the Bellman residuals. Expanding the expression for \(r_i(x)\) gives
\begin{equation}
\nonumber
r_i(x)=\min _u\left\{l_i(u)+\widehat{V}_{i+1}\left(\Psi_i^{\theta}(x, u)\right)\right\}-\widehat{V}_i(x).
\end{equation}
For \(x\in \widehat{\mathcal{S}}_i\), pick a feasible \(u\) with \(\|u\|_2\le \kappa_i\). We have
\begin{equation}
\nonumber
\begin{aligned}
r_i(x)
& \leq l_i(u)+C\left(\mathcal{M}_{t_{i+1}}^\theta\left(\Psi_i^\theta(x, u)\right)\right)-C\left(\mathcal{M}_{t_i}^\theta(x)\right) \\
& \leq \frac{\lambda_{\textup{oc}}}{2} \kappa_i^2 \Delta t_i+L_C\left(\varepsilon_i+L_{\mathcal{M}_x,i+1}\|u\|_2 \Delta t_i\right) \\
& \leq L_C \varepsilon_i+\left(\frac{\lambda_{\textup{oc}}}{2} \kappa_i^2+L_C L_{\mathcal{M}_x,i+1} \kappa_i\right) \Delta t_i.
\end{aligned}
\end{equation}
For a lower bound, we use
\begin{equation}
\nonumber
\begin{aligned}
C\left(\mathcal{M}_{t_{i+1}}^\theta\left(\Psi_i^\theta(x, u)\right)\right) &\geq C\left(\mathcal{M}_{t_{i+1}}^\theta\left(\Psi_i^\theta(x)\right)\right)\\&\quad -L_C L_{\mathcal{M}_x,i+1}\|u\|_2 \Delta t_i.
\end{aligned}
\end{equation}
Thus,
\begin{equation}
\nonumber
\begin{aligned}
r_i(x)
&\geq \min _u\left\{\frac{\lambda_{\textup{oc}}}{2}\|u\|_2^2 \Delta t_i-L_C L_{\mathcal{M}_x,i+1}\|u\|_2 \Delta t_i\right\}-L_C \varepsilon_i \\
&=-\frac{L_C^2 L_{\mathcal{M}_x,i+1}^2}{2 \lambda_{\textup{oc}}} \Delta t_i-L_C \varepsilon_i.
\end{aligned}
\end{equation}
Combining both sides gives
\begin{equation}
\nonumber
\left\|r_i\right\|_{\infty, \widehat{\mathcal{S}}_i} \leq L_C \varepsilon_i+\Gamma_i \Delta t_i.
\end{equation}
Therefore,
\begin{equation}
\nonumber
0 \leq \mathcal{J}\left(x_0,u^{\textup{P3}}\right)-\mathcal{J}\left(x_0,u^{\textup{P2}}\right) \leq 2 \sum_{i=1}^{N-1}\left(L_C \varepsilon_i+\Gamma_i \Delta t_i\right),
\end{equation}
which proves \eqref{mpc_suboptimality_bound}.
\end{proof}

\subsection{Proof of Theorem~\ref{equivalence_theorem}}

\begin{proof}
The equivalence between Problem~\ref{receding_horizon_problem_original} and Problem~\ref{receding_horizon_problem_simplified} follows directly from the control-to-state change of variables used to derive Problem~\ref{receding_horizon_problem_simplified}. The equivalence between Problem~\ref{receding_horizon_problem_simplified} and Problem~\ref{receding_horizon_problem_reversed} follows from the invertible reparameterization \(\widehat{x}_N=\mathcal{M}_{t_{i+1}}^\theta(x_{i+1})\) used to derive Problem~\ref{receding_horizon_problem_reversed}. These transformations preserve the objective, the feasible set, and the induced update at each step. Therefore, all three formulations produce the same control trajectory \(\{u_i\}_{i=0}^{N-1}\) and state trajectory \(\{x_i\}_{i=0}^N\).
\end{proof}

\subsection{Proof of Theorem~\ref{fixed_point_error_theorem}}

\begin{proof}
\(C(y)\) cancels, so only the quadratic terms differ. Recall \eqref{fixed_point_operator} and \eqref{fixed_point_inverse}. Set \(x^*(y)=\left(\mathcal{M}_{t_{i+1}}^\theta\right)^{-1}(y)\) and \(x^{(0)}=\bar{x}_{i+1}\). Taking one-step fixed-point iteration, we have
\begin{equation}
\nonumber
\begin{aligned}
x^{(1)}(y)&=T_{t_{i+1}}^y \left(x^{(0)}\right)=\alpha_{t_{i+1}} y+\beta_{t_{i+1}} \mathcal{W}_{t_{i+1}}^\theta\left(\bar{x}_{i+1}\right)\\&=\mathcal{F}_{t_{i+1}}^\theta(y).
\end{aligned}
\end{equation}
Since by assumption \(\left|\beta_{t_{i+1}}\right| L_{\mathcal{W}_x,i+1}<1\), the operator \(T_{t_{i+1}}^y(\cdot)\) is a contraction with factor \(r=\left|\beta_{t_{i+1}}\right| L_{\mathcal{W}_x,i+1}\). Hence
\begin{equation}
\nonumber
\begin{aligned}
&\|x^{(1)}(y)-x^*(y)\|_2
=\left\|T_{t_{i+1}}^y\left(x^{(0)}\right)-T_{t_{i+1}}^y\left(x^*\right)\right\|_2 \\
&\le r \|x^{(0)}-x^*(y)\|_2 \\
&\le r\left(\left\|x^{(0)}-x^{(1)}(y)\right\|_2+\left\|x^{(1)}(y)-x^*(y)\right\|_2\right).
\end{aligned}
\end{equation}
Therefore,
\begin{equation}
\nonumber
\|x^{(1)}(y)-x^*(y)\|_2 \le \frac{r}{1-r} \|x^{(0)}-x^{(1)}(y)\|_2.
\end{equation}

Using \eqref{posterior_identity} at \(\bar{x}_{i+1}\), we have \(x^{(1)}(y)-x^{(0)}=\alpha_{t_{i+1}}(y-\bar{y})\), where \(\bar{y}=\mathcal{M}_{t_{i+1}}^{\theta}(\bar{x}_{i+1})\). Let \(a=x^{*}(y)-x^{(0)}\), \(b=x^{(1)}(y)-x^{(0)}\), and \(c=b-a=x^{(1)}(y)-x^{*}(y)\). Then
\begin{equation}
\nonumber
\begin{aligned}
\left|\|a\|_2^2-\|b\|_2^2\right|
&=\left|\|b-c\|_2^2-\|b\|_2^2\right| \\
&=\left|-2\langle b, c\rangle+\|c\|_2^2\right| \\
&\leq 2\|b\|_2\|c\|_2+\|c\|_2^2.
\end{aligned}
\end{equation}
With \(\|b\|_2=\left|\alpha_{t_{i+1}}\right|\|y-\bar{y}\|_2\) and \(\|c\|_2\le \frac{r}{1-r}\|b\|_2\), we obtain
\begin{equation}
\nonumber
\begin{aligned}
\left|\|a\|_2^2-\|b\|_2^2\right|
&\leq \alpha_{t_{i+1}}^2\|y-\bar{y}\|_2^2\left(\frac{2 r}{1-r}+\frac{r^2}{(1-r)^2}\right) \\
&=\alpha_{t_{i+1}}^2\|y-\bar{y}\|_2^2 \frac{2 r-r^2}{(1-r)^2}.
\end{aligned}
\end{equation}
Multiplying by \(\frac{\lambda_{\textup{oc}}}{2\Delta t_i}\) yields \eqref{fixed_point_objective_bound} as claimed.
\end{proof}

\subsection{Proof for Remark~\ref{continuous_time_limit}}

\begin{proof}
Assume in addition that $t\mapsto \mathcal M_{t}^\theta(x)$ is $L_{\mathcal{M}_t,i}$-Lipschitz on the region of interest (uniformly in $x$), and that the velocity is bounded on \(\widehat{\mathcal{S}}_i\). Then
\begin{equation}
\nonumber
\begin{aligned}
\varepsilon_i
=& \sup _{x \in \widehat{\mathcal{S}}_i}\left\|\mathcal{M}_{t_{i+1}}^\theta\left(\Psi_i^\theta(x)\right)-\mathcal{M}_{t_i}^\theta(x)\right\|_2 \\
\leq& \sup _{x \in \widehat{\mathcal{S}}_i}\left\|\mathcal{M}_{t_{i+1}}^\theta\left(\Psi_i^\theta(x)\right)-\mathcal{M}_{t_{i+1}}^\theta(x)\right\|_2 \\
& +\sup _{x \in \widehat{\mathcal{S}}_i}\left\|\mathcal{M}_{t_{i+1}}^\theta(x)-\mathcal{M}_{t_i}^\theta(x)\right\|_2 \\
\leq& \left(L_{\mathcal{M}_x,i+1}\sup _{x \in \widehat{\mathcal{S}}_i}\left\|v_{t_i}^\theta(x)\right\|_2+L_{\mathcal{M}_t,i}\right)\Delta t_i \\
=& \mathcal{O}\left(\Delta t_i\right).
\end{aligned}
\end{equation}
\end{proof}

\begin{revisedblock}

\section{Methodological Discussion}
\label{app:methodological_discussion}

\subsection{Minimal-Intervention Principle}
\label{app:minimal_intervention}

A central design principle behind \algname\ is \emph{minimal intervention}. The pretrained flow-matching model encodes a useful prior over the data distribution, including structural and behavioral properties that are implicitly captured by the training dataset. Since the goal of generative modeling is to sample from this underlying distribution, any intervention introduced by our method should be as limited as possible, and should modify the output only when truly necessary, for example, to satisfy hard constraints or improve the task objective.

Therefore, given the cost \(C(x)\) and constraint \(h(x)\le 0\), our goal is not merely to solve an abstract constrained optimization problem over samples. Rather, we aim to generate samples that remain faithful to the training data distribution while also satisfying downstream requirements. In many applications, the explicit cost and constraint functions do not fully specify what makes a sample desirable. The pretrained model, therefore, provides important additional priors that should be preserved whenever possible. In robotics, such priors include demonstration-consistent behavior and physical plausibility. In image editing, they include identity, pose, lighting, texture, and other visual attributes. This is the reason we introduce control regularization, even though the control here does not correspond to an energy-based actuation cost in a real-world physical system. Without a regularization term that keeps the adjusted distribution close to the nominal one, the optimizer may exploit loopholes in the cost and constraint specification, producing spurious samples that are feasible and low-cost, yet nevertheless unrealistic or undesirable.

This interpretation becomes explicit after the control-to-state change of variables. Let $\bar{x}_{i+1}=x_i+v_{t_i}^{\theta}(x_i)\Delta t_i$, which denotes the nominal next state obtained by following the pretrained model without any intervention. The actual next state is then obtained by adding the control input to this nominal update: $x_{i+1}=\bar{x}_{i+1}+u_i\Delta t_i$.
Then we obtain
\begin{equation}
\nonumber
\frac{\lambda_{\textup{oc}}}{2}\lVert u_i\rVert_2^2\Delta t_i
=
\frac{\lambda_{\textup{oc}}}{2\Delta t_i}\lVert x_{i+1}-\bar{x}_{i+1}\rVert_2^2 .
\end{equation}
Therefore, penalizing large control effort is exactly equivalent to penalizing large deviations from the nominal sampler at each step.

Empirically, there is a trade-off between steering aggressively toward a better task objective and remaining close to the nominal generative trajectory. If the regularization is too small, the optimization tends to over-exploit the explicit cost and constraint specification, which can lead to spurious samples and degraded performance. If the regularization is too large, the method becomes conservative and may not improve the task objective sufficiently. In practice, we find it particularly important to avoid setting the regularization too small, since over-exploitation is often more detrimental than mild conservatism.

Finally, we note that our proposed framework is general and naturally subsumes several important special cases. If no task cost is needed, one may set $C(x)=0$, yielding a pure minimal-intervention safety correction: the sampler remains unchanged unless intervention is required to satisfy the constraint. If no constraint is needed, one may simply remove $h(x)\leq 0$, which yields reward optimization regularized toward the nominal sampler. If the cost and constraint are believed to provide a sufficiently complete specification, and over-optimization is not a concern, one may set $\lambda_{\textup{oc}}=0$, thereby recovering an unregularized constrained optimization formulation.

\subsection{One-Step Decomposition}
\label{app:one_step_horizon}

The one-step receding-horizon formulation is a deliberate design choice. The original full-horizon problem is challenging not only because of the number of variables, but also because terminal constraints must be propagated backward through the neural dynamics, which induces an extremely complicated feasible set. A longer MPC horizon would only partially reduce the computational burden while largely preserving this coupling difficulty.

More importantly, the subsequent transformations that make \algname\ practical---namely, the control-to-state change of variables, reverse reparameterization, and one-step fixed-point approximation---rely on the additive one-step structure
\[
x_{i+1}=x_i+v_{t_i}^{\theta}(x_i)\Delta t_i+u_i\Delta t_i.
\]
A multi-step horizon would reintroduce coupling across multiple steps and break this tractable structure. In this sense, the one-step horizon is precisely what enables optimization to be carried out directly in terminal-state space while avoiding the distorted preimage induced by the neural dynamics. Although this decomposition is local at each step, the resulting algorithm is not purely myopic, since each subproblem still evaluates proxies for the terminal cost and constraint through the posterior estimate \(\mathcal{M}_t^\theta(x)\approx \mathbb{E}[X_1\mid X_t=x]\).

\subsection{Feasibility, Stability, and Efficiency}
\label{app:feasibility_stability_efficiency}

Regarding feasibility, we would like to emphasize two aspects. First, although we use the posterior estimate \(\mathcal{M}_t^\theta(x)\) to construct proxies for the terminal cost and constraint, the actual algorithm does not require this estimator to be accurate in order to produce feasible samples. At the final time step \(t_N\), the posterior estimator becomes exact unconditionally, i.e., \(\mathcal{M}_{t_N}^{\theta}(x_N)=x_N\), so the constraint is enforced exactly at the terminal state. This property is formalized in Proposition~\ref{safety_guarantee}. Second, \algname\ is designed specifically to avoid the distorted feasible sets that arise when directly solving the full-horizon optimal control problem, or when optimizing directly over \(x_{i+1}\) subject to the pulled-back constraint \(h(\mathcal{M}_{t_{i+1}}^\theta(x_{i+1}))\le 0\). Through reverse reparameterization and a one-step fixed-point approximation, the optimization is instead carried out directly over the predicted terminal state \(\widehat{x}_N\), so the constraint is imposed in the natural terminal-sample space. In practice, this is much easier to handle numerically, especially when the original terminal constraint is simple. This formulation also isolates the constrained optimization subproblem, allowing it to be solved using off-the-shelf solvers without requiring custom modifications to handle the neural dynamics. If the constraint or cost has additional structure, one may further design custom solvers that exploit this structure to improve efficiency. Additionally, the nominal terminal prediction \(\bar{x}_N\) provides a natural warm start for optimization.

That said, the constrained sampling problem can still become challenging in pathological regimes. Performance may degrade when the pretrained flow-matching model is poor, so that the posterior estimate \(\mathcal{M}_{t}^\theta(\cdot)\) becomes unstable, or when the feasible set is extremely irregular. Such regimes are difficult for essentially any optimization-based approach. Empirically, however, we find that \algname\ is robust across all tasks in our experiments, all of which involve highly nonlinear and nonconvex constraints. This suggests that our method is effective in practice for reasonably challenging problems.

In terms of efficiency, solving a constrained optimization problem at each sampling step can sometimes be time-consuming. However, this is consistent with the intended positioning of our algorithm. \algname\ is a training-free, plug-and-play method for pretrained models, and the additional computation is inherent to such approaches, since constraint enforcement is carried out at inference time rather than absorbed into model training. An alternative would be to fine-tune the model to satisfy the constraints, but this can be substantially more expensive and may introduce additional stability issues during training. Empirically, we find that it is not necessary to solve the constrained optimization problem at every sampling step. In the early stages of sampling, both the posterior estimator and the fixed-point approximation are less accurate, so optimization at those stages is often unnecessary. Instead, we can skip the early steps and activate constrained optimization only in the later stages, which can achieve a good balance between efficiency and performance.

\end{revisedblock}

\section{Experiment Details}
\label{app:experiment_details}

In this section, we provide a detailed description of the experimental settings and implementation choices for all four tasks.

\subsection {Robotic Manipulation}

The flow-matching model is a temporal U-Net following \cite{janner2022planning}, trained on the D3IL dataset \cite{jia2024towards} with training details as in \cite{feng2025on}. The trajectory horizon is \(H=16\) and the replanning horizon is \(T=8\). Obstacle-avoidance constraints for the red pillars and purple regions follow \cite{romer2025diffusion}. The dynamics constraints are \(s_{i+1} = A s_i + B a_i + c\) for \(i=0,\ldots,H-2\), where \(A\in\mathbb{R}^{4\times 4}\), \(B\in\mathbb{R}^{4\times 2}\), and \(c\in\mathbb{R}^{4}\) are fitted via least squares on the training data. The cost \(C(\cdot)\) is the squared distance from \(s_{H-1}\) to the target region. During sampling, we use \(N=10\) discretization steps for all methods. OC-Flow is implemented following \cite{wang2025training}. Given a cost function \(\widehat{C}(\cdot)\), at sampling step \(i\) Gradient Guidance updates the sample \(x\) using the gradient \(\nabla_x \widehat{C}(x + (1 - t_i)\,v_{t_i}^\theta(x))\). We use \(\widehat{C}(\cdot) = C(\cdot) + C_{\textup{penalty}}(\cdot)\) for Gradient Guidance and OC-Flow, where \(C_{\textup{penalty}}(\cdot)\) applies quadratic penalties to constraint violations. In Projection-All, we project at every sampling step; in Projection-Late, we project during the second half of the steps; and in Projection-Relaxed, the augmented Lagrangian implementation follows \cite{zhang2025constrained}. When combining projection-based methods with Gradient Guidance, the guidance cost is the original \(C(\cdot)\) (without penalties), since constraint enforcement is handled by projection. In \algname, we solve \eqref{final_optimization_formulation} only during the second half of the sampling steps, since the posterior estimates are less accurate early in sampling (see Remark~\ref{practical_heuristic}). All projection operations in related baselines and the optimization problem \eqref{final_optimization_formulation} in \algname\ are solved using IPOPT with default settings.

\subsection{Maze Navigation}

The trajectory horizon is \(H=384\). Obstacle-avoidance constraints for the red regions follow \cite{xiao2025safediffuser}. The cost \(C(x)\) is the (squared) total length of the trajectory. The flow-matching model is trained on the D4RL maze2d-large-v1 dataset \cite{fu2020d4rl}. All other implementation details are identical to those in the robotic manipulation task.

\subsection{PDE Control}

The flow-matching model is a temporal U-Net, with architecture and dataset following \cite{hu2025from}. The temporal domain is \(t \in [0,1]\) and the spatial domain is \(s \in [0,1]\) (i.e., \(T=L=1\)). The spatiotemporal grid uses \(m=10\) time steps and \(n=128\) spatial points. The time-varying state constraints are defined as \(|u(t,s)| \leq 0.8\cdot(2t^2-2t+1)\). The dynamics constraints come from a finite-difference discretization of the Burgers' PDE, with viscosity \(\nu\) treated as uncertain within \([\nu_{\textup{min}},\nu_{\textup{max}}]=[0.0, 0.02]\). The cost \(C(\cdot)\) is the total control energy, i.e., the sum of squared control values over all \((t,s)\) grid points. All other implementation details are identical to those in the robotic manipulation task.

\subsection{Text-Guided Image Editing}

We adopt the full pipeline of \cite{wang2025training}, including the flow-matching model, the CLIP and LPIPS models, and image preprocessing.  Let \(f(x)=\textup{CLIP}(x)\) and \(g(x)=\textup{LPIPS}(x, x_{\textup{ref}})\), where \(x_{\textup{ref}}\) is the input image. The cost is \(C(x)=-f(x)\). The constraint is \(g(x) \leq \epsilon\), where \(\epsilon=0.06\) denotes the LPIPS upper bound. During sampling, we use \(N=100\) discretization steps for all methods. For \algname, we solve \eqref{final_optimization_formulation} with an augmented Lagrangian method. At sampling step \(i\), the augmented Lagrangian is \(\mathcal{L}_i(x, \lambda, \rho) = -f(x) + f_{\textup{reg}}(x) - \lambda (g(x)-\epsilon) - \frac{\rho}{2} (\max\left\{0, g(x)-\epsilon\right\})^2\), where \(f_{\textup{reg}}(x)\) denotes the regularization term corresponding to \(\alpha_{t_{i+1}}^2\left\|\widehat{x}_N-\bar{x}_N\right\|_2^2\) in \eqref{final_optimization_formulation}. Since posterior estimates are less accurate in early steps (see Remark~\ref{practical_heuristic}), we omit the LPIPS constraint during the first half of the sampling steps (i.e., forcing \(\mathcal{L}_i = -f(x)+ f_{\textup{reg}}(x)\) for \(i < N/2\)). We run 40 gradient-ascent steps on \(\mathcal{L}_i\) per sampling step \(i\) with the same learning rate \(\eta=2.5\) as OC-Flow, during which \(\lambda\) and \(\rho\) are updated every 8 steps. For OC-Flow and Gradient Guidance, we use the guidance cost \(\widehat{C}(x) = -f(x) + (\max\left\{0, g(x)-\epsilon\right\})^2\), which encodes the LPIPS constraint as a fixed-weight penalty. For Projection-Relaxed + Gradient Guidance, we follow \cite{zhang2025constrained} to enforce the LPIPS constraint \(g(x) \leq \epsilon\), while using \(-f(x)\) as the guidance cost.

\begin{revisedblock}

\section{Additional Visualizations}
\label{app:steering_visualizations}

To better demonstrate how our method iteratively steers the sampling trajectory toward a feasible, low-cost sample, we provide additional visualizations of the generation process for \algname\ across all four tasks, as shown in Fig.~\ref{fig:app_visual_maze}, Fig.~\ref{fig:app_visual_d3il}, Fig.~\ref{fig:app_visual_pde}, and Fig.~\ref{fig:app_visual_image}.

In all figures below, the first row shows the actual intermediate samples \(x_t\) generated by \algname, while the second row shows the predicted terminal samples \(\mathcal{M}_{t}^{\theta}(x_t)\approx \mathbb{E}[X_1\mid X_t=x_t]\). Each column corresponds to a different sampling time step. These visualizations illustrate the central innovation of \algname: the constraints are imposed on the predicted terminal sample rather than directly on the noisy intermediate state.

\begin{figure*}[htbp]
    \centering
    \includegraphics[width=\textwidth]{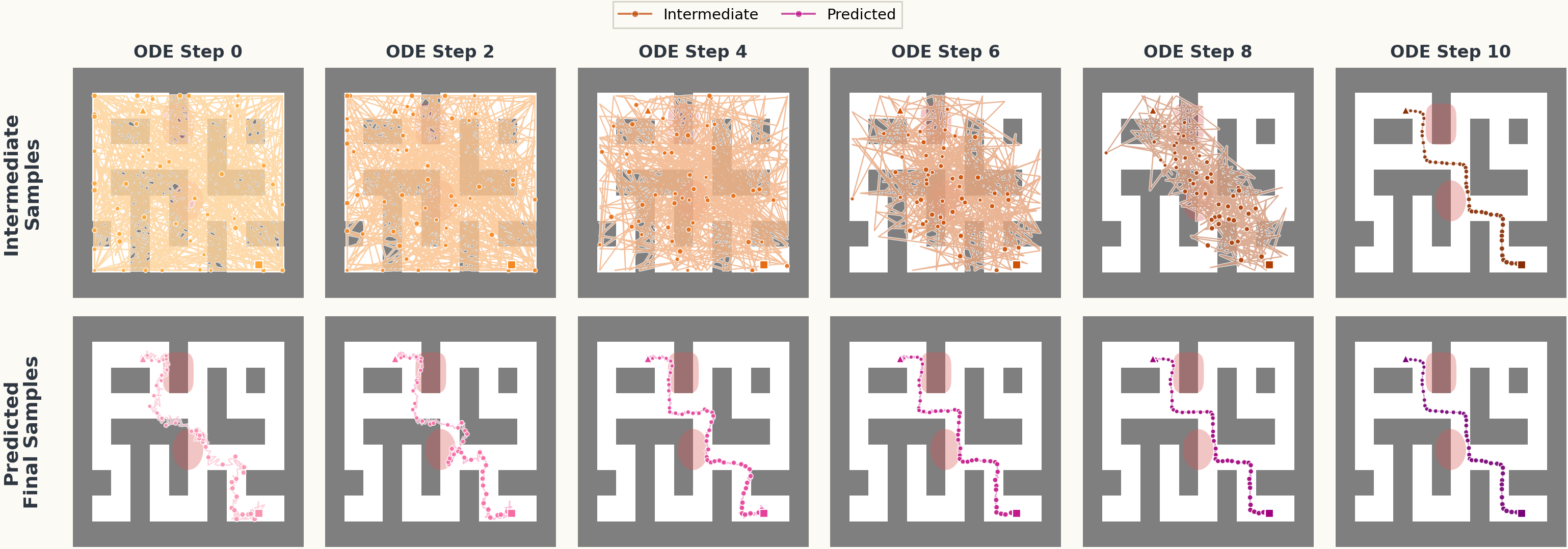}
    \caption{\revised{Visualization of the generation process from \algname\ in the maze navigation task. The planned trajectory contains \(384\) points; for readability, we visualize a downsampled trajectory with \(50\) points.}}
    \label{fig:app_visual_maze}
\end{figure*}

\begin{figure*}[htbp]
    \centering
    \includegraphics[width=\textwidth]{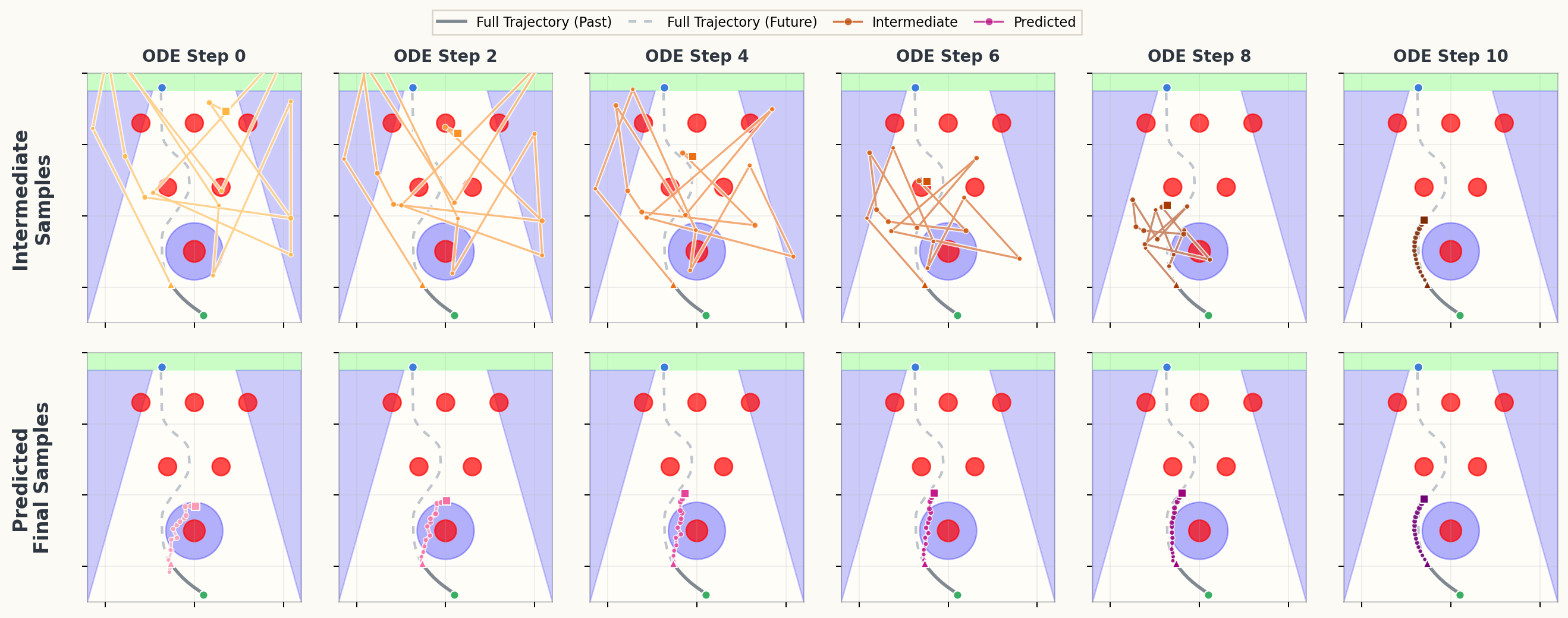}
    \caption{\revised{Visualization of the generation process from \algname\ in the robotic manipulation task. Since the policy replans in a receding-horizon manner, we show one representative planning instance during execution.}}
    \label{fig:app_visual_d3il}
\end{figure*}

\begin{figure*}[htbp]
    \centering
    \includegraphics[width=\textwidth]{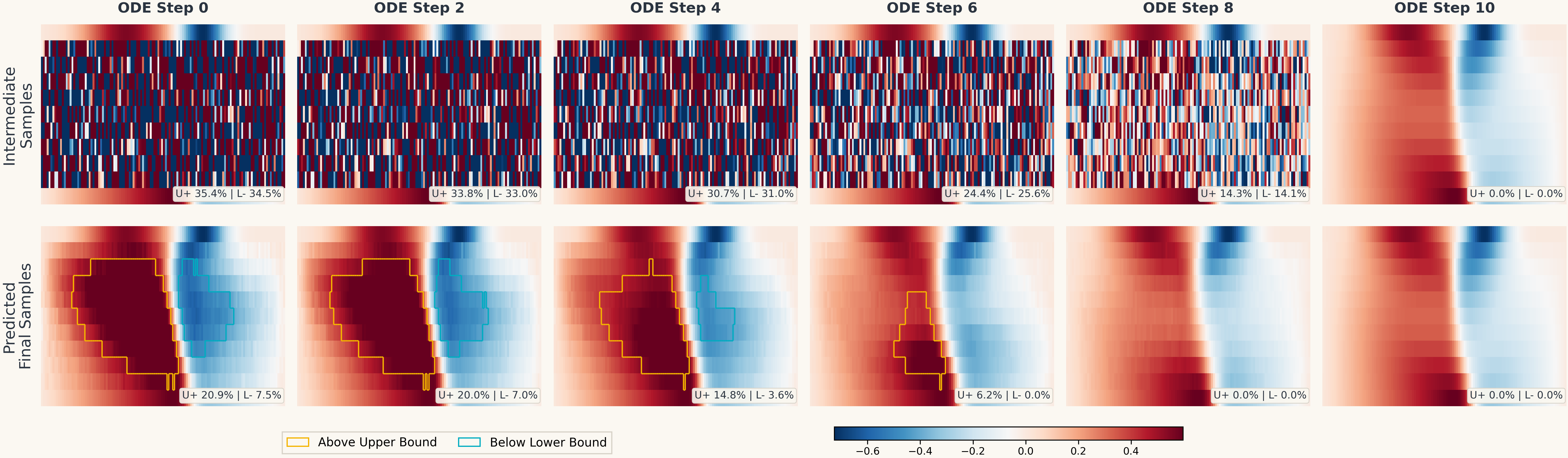}
    \caption{\revised{Visualization of the generation process from HardFlow in the PDE control task. A panel of the form $U+ 20.9\% \mid L- 7.5\%$ indicates that 20.9\% of the spatiotemporal grids exceed the upper constraint bound, while 7.5\% fall below the lower bound. For the predicted terminal samples, we additionally annotate the image with markers indicating the locations of these violations.}}
    \label{fig:app_visual_pde}
\end{figure*}

\begin{figure*}[htbp]
    \centering
    \includegraphics[width=\textwidth]{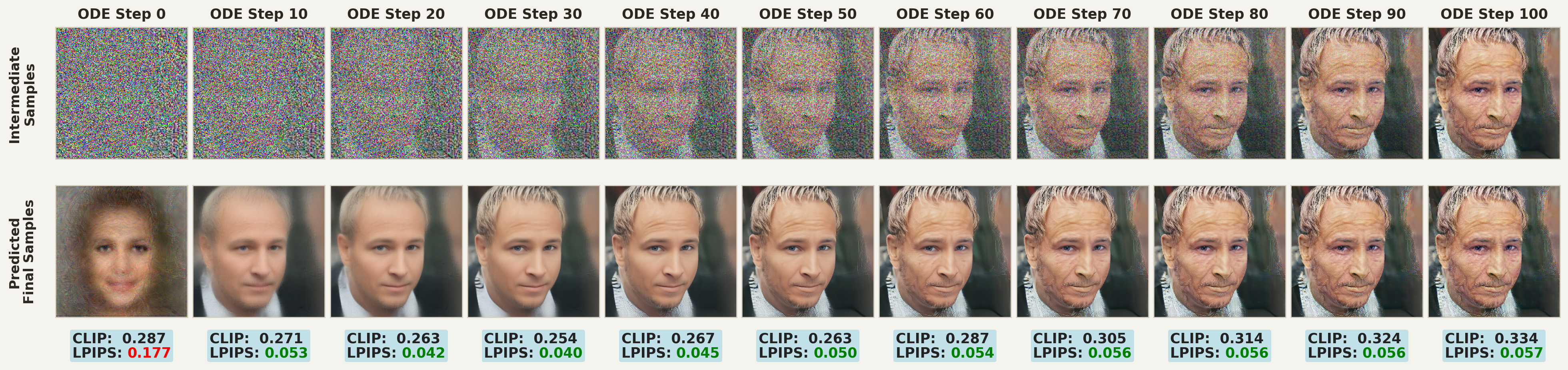}
    \caption{\revised{Visualization of the generation process from \algname\ in the text-guided image editing task. Below each predicted terminal sample, we report CLIP and LPIPS. LPIPS values that satisfy the hard constraint are displayed in green, and violations are displayed in red.}}
    \label{fig:app_visual_image}
\end{figure*}

\section{Additional Baselines}
\label{app:additional_baselines}

\subsection{Comparison with Post-Processing Methods}
\label{app:posthoc_baselines}

We additionally compare against two post-processing baselines on the robotic manipulation task.

\noindent\textbf{Post-Hoc Filtering.}
We sample a batch of \(64\) trajectories from the nominal flow-matching model, discard the infeasible ones, and then select the feasible sample with the lowest cost. This baseline tests whether it is sufficient to rely on the pretrained model to generate feasible samples by chance and then simply choose the best one among them.

\noindent\textbf{Post-Hoc Optimization.}
After sampling from the pretrained model, we take the terminal sample \(x_N\) as the initialization for the following optimization problem:
\begin{equation}
\nonumber
\begin{aligned}
	\min_{x_N} \quad & C(x_N)\\
	\text{s.t.} \quad & h(x_N) \leq 0,
\end{aligned}
\end{equation}
This baseline tests whether it is sufficient to modify only the final sample, without preserving the implicit prior encoded by the pretrained model.

The results are shown in Table~\ref{tab:app_posthoc_baselines}. Both post-processing baselines perform substantially worse than \algname. Post-hoc filtering depends on the nominal sampler producing feasible trajectories by chance, which is very difficult when the constraints are not considered in the training dataset. The post-hoc optimization baseline starts from the output of the nominal sampler, but then optimizes only the explicit cost and constraints, without preserving the pretrained model prior during the optimization process. As a result, it tends to over-exploit aspects of the cost and constraint formulation that do not fully capture the underlying structure of the data (e.g., physical consistency), thereby producing spurious samples. In particular, although the planned trajectories satisfy the constraints, the robot cannot reliably execute them because of their physical inconsistency, which in turn leads to a low realized safety rate.

\begin{table*}[htbp]
    \centering
    \ifmarked\color{blue}\fi
    \caption{\revised{Comparison with post-processing baselines on the robotic manipulation task. Values of the form \(a\pm b\) indicate mean \(\pm\) standard deviation.}}
    \label{tab:app_posthoc_baselines}
    \setlength{\tabcolsep}{6pt}
    \begin{tabular}{l c c c}
        \toprule
        \textbf{Method} & \textbf{Safety Rate} & \textbf{Total Steps (Safe Trials)} & \textbf{Computation Time (s)} \\
        \midrule
        \algname\ (ours) & 1.00 & 52.5 $\pm$ 4.4 & 0.190 $\pm$ 0.023 \\
        Post-hoc optimization & 0.04 & 57.5 $\pm$ 2.1 & 0.054 $\pm$ 0.015 \\
        Post-hoc filtering & 0.50 & 63.2 $\pm$ 5.8 & 0.059 $\pm$ 0.005 \\
        \bottomrule
    \end{tabular}
\end{table*}

\subsection{Comparison with SafeFlowMatcher}
\label{app:safeflowmatcher}

We also compare against a concurrent work, SafeFlowMatcher \cite{yang2025safeflowmatcher}.
Conceptually, SafeFlowMatcher and our method share a similar motivation: both avoid enforcing constraints on intermediate noisy states, which can overly restrict the sampling path and thereby harm sample quality. In our method, we adopt a trajectory optimization perspective by formulating constrained sampling as a constrained optimal control problem, with the constraints imposed at the terminal time step. SafeFlowMatcher, by contrast, uses a two-phase prediction-correction scheme. In the prediction phase, it first generates a coarse sample from the pretrained model, for example, by using only a small number of time steps. In the correction phase, it then defines a vanishing time-scaled flow dynamics and applies a CBF-based correction to steer this dynamics toward the safe set.

Compared with SafeFlowMatcher, our method offers several advantages. First, we propose a general hard-constrained optimization framework that applies to a broader range of constrained generation problems beyond safe planning for robotics, as demonstrated by our experiments on PDE control and text-guided image editing. Second, our method incorporates both hard constraints and auxiliary costs within a unified optimization framework, whereas SafeFlowMatcher focuses primarily on enforcing safety constraints and may therefore lead to suboptimal task performance when additional objectives are important. Third, the CBF-based correction in SafeFlowMatcher requires careful tuning of the CBF parameters and the relaxation term, and may be less well suited to handling complex or equality constraints. In contrast, our method decouples the constrained optimization from the neural network dynamics through a sequence of principled problem transformations, which allows complex constraints to be handled effectively using off-the-shelf optimization solvers. Fourth, the vanishing time-scaled flow dynamics together with the CBF-based correction generally require fine-grained time discretization to ensure stability and feasibility, which can increase computational cost.

Empirically, we have implemented SafeFlowMatcher and evaluated it on the two robotics benchmarks in our paper: maze navigation and robotic manipulation, which are the tasks aligned with the scope of SafeFlowMatcher. The results are summarized in Tables~\ref{tab:safe_flow_matcher_maze} and~\ref{tab:safe_flow_matcher_d3il}. In the maze navigation task, both methods achieve a perfect safety rate. HardFlow attains a slightly higher task score, as well as a lower computation time. In the robotic manipulation task, HardFlow achieves perfect safety and completes the task efficiently, whereas SafeFlowMatcher fails to provide safe solutions. This is because the manipulation task involves multiple nonconvex obstacles, resulting in a highly complex feasible set that is difficult for the CBF-based correction to handle effectively.

\begin{table*}[htbp]
    \centering
    \ifmarked\color{blue}\fi
    \caption{\revised{Comparison with SafeFlowMatcher on the maze navigation task. Values of the form \(a\pm b\) indicate mean \(\pm\) standard deviation.}}
    \setlength{\tabcolsep}{6pt}
    \begin{tabular}{l c c c c}
        \toprule
        \textbf{Method} & \textbf{Safety Rate} & \textbf{Violations} & \textbf{Score} & \textbf{Computation Time (s)} \\
        \midrule
        HardFlow & 1.00 & 0.0 $\pm$ 0.0 & 1.620 $\pm$ 0.010 & 4.09 $\pm$ 0.82  \\
        SafeFlowMatcher & 1.00 & 0.0 $\pm$ 0.0 & 1.60 $\pm$ 0.00 & 7.61 $\pm$ 0.19 \\
        \bottomrule
    \end{tabular}
    \label{tab:safe_flow_matcher_maze}
\end{table*}

\begin{table*}[htbp]
    \ifmarked\color{blue}\fi
    \centering
    \caption{\revised{Comparison with SafeFlowMatcher on the robotic manipulation task. Values of the form \(a\pm b\) indicate mean \(\pm\) standard deviation.}}
    \setlength{\tabcolsep}{6pt}
    \begin{tabular}{l c c c}
        \toprule
        \textbf{Method} & \textbf{Safety Rate} & \textbf{Total Steps (Safe Trials)} & \textbf{Computation Time (s)} \\
        \midrule
        HardFlow & 1.00 & 52.5 $\pm$ 4.4 & 0.190 $\pm$ 0.023 \\
        SafeFlowMatcher & 0.00 & N/A & 6.70 $\pm$ 0.32 \\
        \bottomrule
    \end{tabular}
    \label{tab:safe_flow_matcher_d3il}
\end{table*}

\section{Sensitivity Analysis and Stress Testing}
\label{app:sensitivity_analysis}

\subsection{Regularization Coefficient}
\label{app:regularization_role}

Let \(\lambda_0\) denote the default regularization coefficient used in \algname. We sweep
\begin{equation}
\nonumber
\lambda_{\textup{oc}}\in\{0.0,\,0.1\lambda_0,\,0.5\lambda_0,\,\lambda_0,\,5\lambda_0,\,10\lambda_0\}
\end{equation}
on the robotic manipulation task.

The results are reported in Table~\ref{tab:reg_sweep}. When $\lambda_{\textup{oc}}$ is too small (e.g., $0.0$ or $0.1\lambda_0$), performance deteriorates substantially. For $\lambda_{\textup{oc}}\in[0.5\lambda_0,10\lambda_0]$, the safety rate remains at $1.00$, while larger values gradually make the method more conservative, resulting in longer trajectories (i.e., more total time steps to reach the goal position). This highlights the importance of the regularization term and further supports the stability of our method, since a wide range of $\lambda_{\textup{oc}}$ values achieves a perfect safety rate while maintaining reasonable performance.

\begin{table*}[htbp]
    \centering
    \ifmarked\color{blue}\fi
    \caption{\revised{Sensitivity to the regularization coefficient $\lambda_{\textup{oc}}$ of \algname\ in the robotic manipulation task. Values of the form \(a\pm b\) indicate mean \(\pm\) standard deviation.}}
    \setlength{\tabcolsep}{6pt}
    \begin{tabular}{c c c c}
        \toprule
        \textbf{$\lambda_{\textup{oc}}$} & \textbf{Safety Rate} & \textbf{Total Steps (Safe Trials)} & \textbf{Computation Time (s)} \\
        \midrule
        $0.0$ & 0.10 & 83.8 $\pm$ 0.45 & 0.63 $\pm$ 0.59 \\
        $0.1\lambda_0$ & 0.64 & 84 $\pm$ 18 & 0.202 $\pm$ 0.020 \\
        $0.5\lambda_0$ & 1.00 & 51.6 $\pm$ 2.2 & 0.192 $\pm$ 0.025 \\
        $\lambda_0$ & 1.00 & 52.5 $\pm$ 4.4 & 0.190 $\pm$ 0.023 \\
        $5\lambda_0$ & 1.00 & 57.2 $\pm$ 7.2 & 0.192 $\pm$ 0.021 \\
        $10\lambda_0$ & 1.00 & 63 $\pm$ 10 & 0.195 $\pm$ 0.023 \\
        \bottomrule
    \end{tabular}
    \label{tab:reg_sweep}
\end{table*}

\subsection{Control Activation Schedule}
\label{app:activation_schedule}

Following Remark~\ref{practical_heuristic}, we study four schedules for activating constrained optimization in the robotic manipulation task: (i) \emph{all}, where we do not skip any steps and apply the constrained optimization at every step; (ii) \emph{early}, where we skip the first 20\% of the total steps and apply the constrained optimization over the remaining 80\%; (iii) \emph{middle}, where we skip the first 50\% of the total steps and apply the constrained optimization over the remaining 50\%, which is also the default schedule used in HardFlow; and (iv) \emph{late}, where we skip the first 80\% of the total steps and apply the constrained optimization only over the final 20\%.

The results are reported in Table~\ref{tab:activation}. When the control is activated too late, the safety rate deteriorates substantially. This is because the intervention occurs too late to steer the trajectory so that the final sample satisfies the constraint. The other three schedules (\emph{all}, \emph{early}, and \emph{middle}) all achieve a perfect safety rate, demonstrating the effectiveness of the proposed method in enforcing hard constraints through trajectory optimization.

Among these three schedules, there is a trade-off between the control activation time and the optimized cost. Activating the control earlier incurs a higher computational burden, while its benefit is limited by the fact that the posterior-mean estimator $\mathcal{M}_t^\theta(x)$ is less accurate in the early stages of sampling. This is reflected in the relatively marginal improvement in the optimized cost (i.e., the total number of steps required for the robot to reach the goal).

Therefore, in practice, we recommend activating the control neither too early nor too late, but rather using an early-to-middle schedule that provides a good balance among constraint satisfaction, cost optimization, and computational efficiency.

\begin{table*}[htbp]
    \centering
    \ifmarked\color{blue}\fi
    \caption{\revised{Sensitivity to the control activation schedule of \algname\ in the robotic manipulation task. Values of the form \(a\pm b\) indicate mean \(\pm\) standard deviation.}}
    \setlength{\tabcolsep}{6pt}
    \begin{tabular}{c c c c}
        \toprule
        \textbf{Control Activation} & \textbf{Safety Rate} & \textbf{Total Steps (Safe Trials)} & \textbf{Computation Time (s)} \\
        \midrule
        All & 1.00 & 50.42 $\pm$ 0.70 & 0.333 $\pm$ 0.023 \\
        Early & 1.00 & 50.6 $\pm$ 1.1 & 0.266 $\pm$ 0.024 \\
        Middle & 1.00 & 52.5 $\pm$ 4.4 & 0.190 $\pm$ 0.023 \\
        Late & 0.84 & 56.5 $\pm$ 5.4 & 0.111 $\pm$ 0.016 \\
        \bottomrule
    \end{tabular}
    \label{tab:activation}
\end{table*}

\subsection{Stress Testing}
\label{app:disjoint_feasible_region}

To further demonstrate the robustness of our method to more complex and irregular geometries, we conduct an additional experiment on the robotic manipulation task, in which we introduce an extra quadrilateral constraint that splits the feasible set into two disconnected parts. The results are shown in Figure~\ref{fig:r4_disjoint} and Table~\ref{tab:r4_disjoint}. We observe that even in this more challenging multi-constraint setting with non-standard geometry, our method still achieves perfect safety while completing the task efficiently.

\begin{figure*}[htbp]
  \centering
  \includegraphics[width=0.36\textwidth]{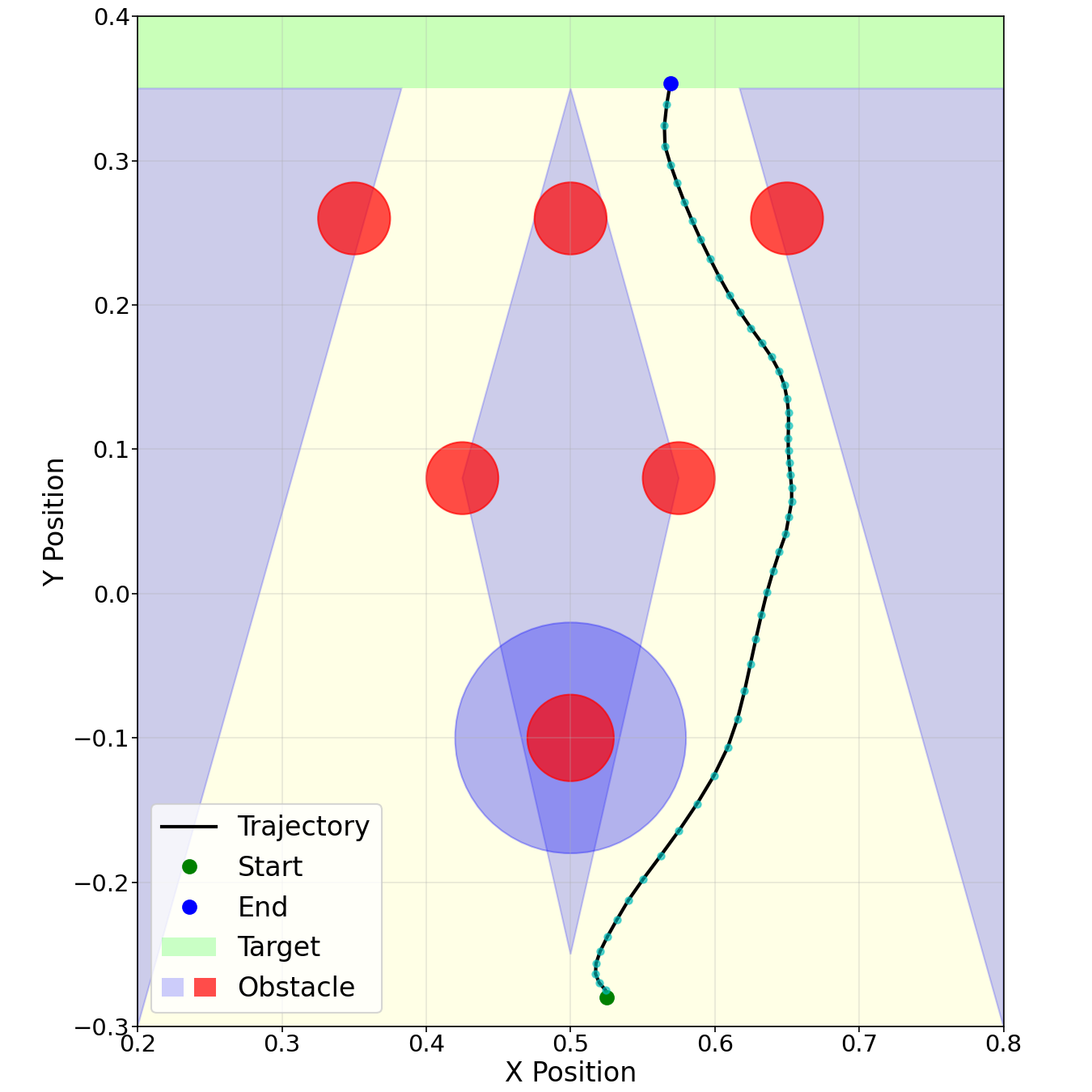}
  \hspace{2em}
  \includegraphics[width=0.36\textwidth]{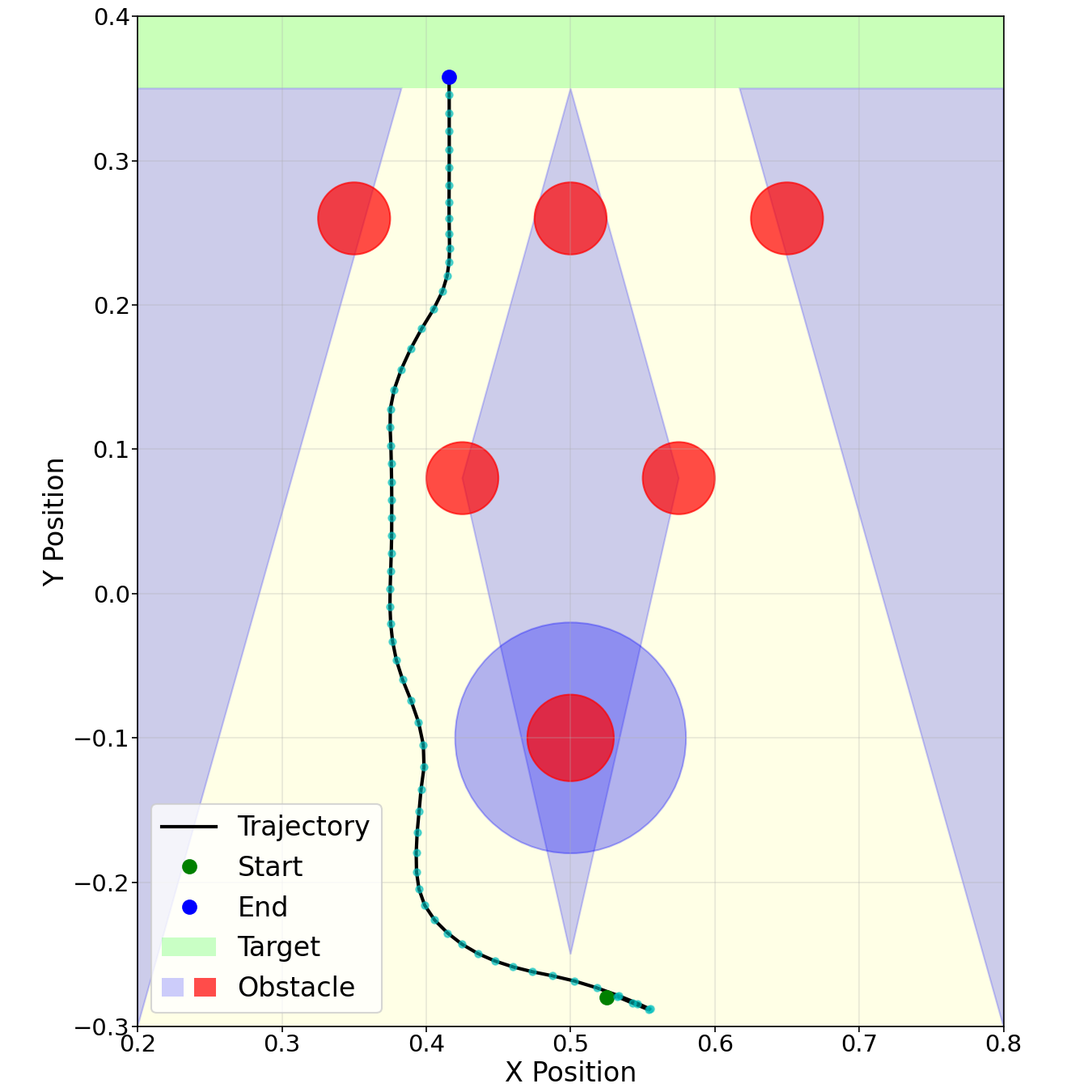}
  \caption{\revised{Visualization of HardFlow in the robotic manipulation task with a more complex feasible region. The left and right panels show two randomly sampled trajectories.}}
  \label{fig:r4_disjoint}
\end{figure*}

\begin{table*}[htbp]
    \centering
    \ifmarked\color{blue}\fi
    \caption{\revised{Performance of HardFlow in the robotic manipulation task with a more complex feasible region. Values of the form \(a\pm b\) indicate mean \(\pm\) standard deviation.}}
    \setlength{\tabcolsep}{6pt}
    \begin{tabular}{l c c c}
        \toprule
        \textbf{Method} & \textbf{Safety Rate} & \textbf{Total Steps (Safe Trials)} & \textbf{Computation Time (s)} \\
        \midrule
        HardFlow & 1.00 & 59.7 $\pm$ 9.8 & 0.31 $\pm$ 0.10 \\
        \bottomrule
    \end{tabular}
    \label{tab:r4_disjoint}
\end{table*}

\section{High-Order Solvers and Non-Uniform Time Grids}
\label{app:solver_schedule_results}

Our method is compatible with both higher-order ODE solvers and non-uniform time grids.

For any solver, let $\Psi_i^{\theta}(x)$ denote the corresponding one-step nominal transition. For the Euler solver, we have
\[
\Psi_i^{\theta}(x) = x + v_{t_i}^{\theta}(x)\Delta t_i.
\]
For the Heun solver, we have
\[
\Psi_i^{\theta}(x) = x + \frac{1}{2}\Bigl(v_{t_i}^{\theta}(x) + v_{t_{i+1}}^{\theta}(x + v_{t_i}^{\theta}(x)\Delta t_i)\Bigr)\Delta t_i,
\]
which requires two evaluations of the flow-matching model $v_t^{\theta}$. Then Problem~2 can be generalized as
\begin{equation}
\nonumber
\begin{array}{@{}l @{\quad} l@{}}
\underset{\left\{x_j\right\}_{j=0}^N, \left\{u_j\right\}_{j=0}^{N-1}}{\min} &
\begin{array}[t]{@{}l@{}}
C(x_N) + \lambda_{\textup{oc}} \sum_{j=0}^{N-1} \frac{1}{2}\left\|u_j\right\|_2^2 \Delta t_j \\
\begin{array}[t]{@{}ll@{}}
\textup{s.t.} & x_0 = \bar{x}_0, \\
              & x_{j+1}=\Psi_j^{\theta}(x_j)+u_j \Delta t_j,\\
              & j=0,1,\cdots,N-1,\\
              & h(x_N) \le 0.
\end{array}
\end{array}
\end{array}
\end{equation}
The same sequence of approximations can be applied to this formulation, leading to the following generalized version of Problem~6:
\begin{equation}
\nonumber
\left\{
\begin{array}{l}
\bar{x}_{i+1}=\Psi_i^{\theta}(x_i), \\[0.3em]
\begin{array}{@{}l @{\quad} l@{}}
\widehat{x}_N^{*}=\underset{\widehat{x}_N}{\operatorname{argmin}} &
\begin{array}[t]{@{}l@{}}
C(\widehat{x}_N) + \frac{\lambda_{\textup{oc}}}{2\Delta t_i}\left\| \mathcal{F}_{t_{i+1}}^{\theta}(\widehat{x}_N)-\bar{x}_{i+1}\right\|_2^2 \\
\begin{array}[t]{@{}ll@{}}
\textup{s.t.} & h(\widehat{x}_N) \le 0.
\end{array}
\end{array}
\end{array} \\[0.3em]
x_{i+1} = \mathcal{F}_{t_{i+1}}^{\theta}(\widehat{x}_N^{*}).
\end{array}
\right.
\end{equation}
Therefore, all results in the Theoretical Analysis section can be extended naturally to the higher-order solver setting.

Regarding non-uniform time grids, our method only assumes a sequence of discretized time steps $\left\{t_i\right\}_{i=0}^N$ and does not impose any specific structural requirement on them. Therefore, non-uniform time grids can be incorporated directly without any modification to the method or its theoretical properties.

To empirically validate the compatibility, we implement two variants of HardFlow on the image editing task: (i) HardFlow with Heun's method, and (ii) HardFlow with the linear-quadratic (LQ) time schedule proposed in \cite{polyak2024movie}. The results are reported in Table~\ref{tab:solver_and_schedule}. Both variants achieve a perfect safety rate (LPIPS below 0.06) while maintaining comparable quality (measured by CLIP score), and also reduce computational cost. These results highlight the flexibility of our method and demonstrate the potential of more advanced solvers and time grids for improving efficiency and quality.

\begin{table*}[htbp]
  \centering
  \ifmarked\color{blue}\fi
  \caption{\revised{Compatibility of HardFlow with higher-order solvers and non-uniform time grids on the image editing task. Values of the form \(a\pm b\) indicate mean \(\pm\) standard deviation.}}
  \setlength{\tabcolsep}{3pt}
  \begin{tabular}{l c c c c}
    \toprule
    \textbf{Method} & \textbf{Safety Rate} & \textbf{LPIPS} & \textbf{CLIP} & \textbf{Computation Time (s)} \\
    \midrule
	HardFlow (Heun) & 1.000 & 0.048 $\pm$ 0.003 & 0.328 $\pm$ 0.024 & 38.40 $\pm$ 0.52 \\
    HardFlow (LQ schedule) & 1.000 & 0.049 $\pm$ 0.003 & 0.291 $\pm$ 0.019 & 36.2 $\pm$ 1.1 \\
    HardFlow (default) & 1.000 & 0.048 $\pm$ 0.002 & 0.317 $\pm$ 0.022 & 51.294 $\pm$ 0.004 \\
    \bottomrule
  \end{tabular}
  \label{tab:solver_and_schedule}
\end{table*}

\end{revisedblock}

\end{document}